%% file: bare_jrnl_compsoc.tex
\documentclass[10pt,journal,compsoc]{IEEEtran}
\usepackage{bm}
\usepackage{amssymb} 
\usepackage{amsthm}
\usepackage{algorithm}
\usepackage{algorithmic}
\usepackage{multirow}
\usepackage{booktabs}
\usepackage{xcolor}         
\usepackage{colortbl}

\newtheorem{proposition}{Proposition}
\newtheorem{assumption}{Assumption}
\newtheorem{lemma}{Lemma}

\usepackage{pifont}
\newcommand{\cmark}{\ding{51}}%
\newcommand{\xmark}{\ding{55}}%

\newcommand{\name}{ProCo}%

\ifCLASSOPTIONcompsoc

  \usepackage[nocompress]{cite}
\else
  \usepackage{cite}
\fi
\ifCLASSINFOpdf
  \usepackage[pdftex]{graphicx}
\else
\fi
\usepackage{amsmath}

\usepackage[pagebackref,breaklinks,colorlinks,bookmarks=false]{hyperref}
\usepackage[capitalize]{cleveref}
\crefname{section}{Sec.}{Secs.}
\Crefname{section}{Section}{Sections}
\crefname{table}{Tab.}{Tabs.}
\Crefname{table}{Table}{Tables}
\crefname{assumption}{Assumption}{Assumptions}

\IEEEpubid{
\fbox{
\parbox{0.975\textwidth}{
\copyright 2024 IEEE. Personal use of this material is permitted. Permission from IEEE must be obtained for all other uses, in any current or future media, including\hfill
reprinting/republishing this material for advertising or promotional purposes, creating new collective works, for resale or redistribution to servers or lists, or reuse of\hfill
any copyrighted component of this work in other works. \vspace{-0.5mm}
}%
}}

\begin{document}

%
\title{Probabilistic Contrastive Learning for Long-Tailed Visual Recognition}

%
%

\author{Chaoqun~Du,
        Yulin~Wang,
        Shiji~Song,
        and Gao~Huang
\IEEEcompsocitemizethanks{\IEEEcompsocthanksitem C. Du, Y. Wang, S. Song, G. Huang are with the Department of Automation, BNRist, Tsinghua University, Beijing 100084, China. Email: \{dcq20, wang-yl19\}@mails.tsinghua.edu.cn, shijis@mail.tsinghua.edu.cn, gaohuang@tsinghua.edu.cn. Corresponding author: Gao Huang.
} 
}

%
%

\markboth{IEEE Transactions on Pattern Analysis and Machine Intelligence}%
{Shell \MakeLowercase{\textit{et al.}}: Bare Demo of IEEEtran.cls for Computer Society Journals}
 



\input{abs}

\maketitle

\IEEEdisplaynontitleabstractindextext

%
\IEEEpeerreviewmaketitle

\input{intro}

\input{relate}

\input{method}

\input{exp}

\input{conclusion}

\input{appendix}

%

%

%

\ifCLASSOPTIONcompsoc
  \section*{Acknowledgments}
\else
  \section*{Acknowledgment}
\fi

This work is supported in part by the National Key R$\&$D Program of China under Grant 2021ZD0140407,
the National Natural Science Foundation of China under Grants 62276150, 42327901.
We also appreciate the generous donation of computing resources by High-Flyer AI.

\ifCLASSOPTIONcaptionsoff
  \newpage
\fi



\bibliographystyle{IEEEtran}
\bibliography{IEEEabrv,IEEEtran}
%
%
%
%
%
%
%

%




%

\end{document}

%% file: abs.tex
\IEEEtitleabstractindextext{%
\begin{abstract}
    Long-tailed distributions frequently emerge in real-world data, where a large number of minority categories contain a limited number of samples.
    Such imbalance issue considerably impairs the performance of standard supervised learning algorithms, which are mainly designed for balanced training sets.
    Recent investigations have revealed that supervised contrastive learning exhibits promising potential in alleviating the data imbalance.
    However, the performance of supervised contrastive learning is plagued by an inherent challenge: it necessitates sufficiently large batches of training data to construct contrastive pairs that cover all categories, yet this requirement is difficult to meet in the context of class-imbalanced data.
    To overcome this obstacle, we propose a novel probabilistic contrastive (\name) learning algorithm that estimates the data distribution of the samples from each class in the feature space, and samples contrastive pairs accordingly.
    In fact, estimating the distributions of all classes using features in a small batch, particularly for imbalanced data, is not feasible.
    Our key idea is to introduce a reasonable and simple assumption that the normalized features in contrastive learning follow a mixture of von Mises-Fisher (vMF) distributions on unit space, which brings two-fold benefits. First, the distribution parameters can be estimated using only the first sample moment, which can be efficiently computed in an online manner across different batches.
   Second, based on the estimated distribution, the vMF distribution allows us to sample an infinite number of contrastive pairs and derive a closed form of the expected contrastive loss for efficient optimization.
    Other than long-tailed problems, \name~can be directly applied to semi-supervised learning by generating pseudo-labels for unlabeled data, which can subsequently be utilized to estimate the distribution of the samples inversely.
    Theoretically, we analyze the error bound of \name.
    Empirically, extensive experimental results on supervised/semi-supervised visual recognition and object detection tasks demonstrate that \name~consistently outperforms existing methods across various datasets.
    Our code is available at \url{https://github.com/LeapLabTHU/ProCo}.
\end{abstract}

\begin{IEEEkeywords}
Long-Tailed Visual Recognition, Contrastive Learning, Representation Learning, Semi-Supervised Learning.
\end{IEEEkeywords}
}

%% file: intro.tex
\IEEEraisesectionheading{\section{Introduction}\label{sec:introduction}}

%
%
%
%

\IEEEPARstart{W}{ith} the accelerated progress in deep learning and the emergence of large, well-organized datasets, significant advancements have been made in computer vision tasks, including image classification \cite{Krizhevsky2017,He2016,Huang2017}, object detection \cite{Ren2015}, and semantic segmentation \cite{Zhao2017}.
The meticulous annotation of these datasets ensures a balance among various categories during their development \cite{Russakovsky2015}.
Nevertheless, in practical applications, acquiring comprehensive datasets that are both balanced and encompass all conceivable scenarios remains a challenge.
Data distribution in real-world contexts often adheres to a long-tail pattern, characterized by an exponential decline in the number of samples per class from head to tail \cite{VanHorn2018}.
This data imbalance presents a considerable challenge for training deep models, as their ability to generalize to infrequent categories may be hindered due to the limited training data available for these classes.
Furthermore, class imbalance can induce a bias towards dominant classes, leading to suboptimal performance on minority classes \cite{Graf2021,Fang2021}. Consequently, addressing the long-tail distribution issue is crucial for the successful application of computer vision tasks in real-world scenarios.

In tackling the long-tailed data conundrum, a multitude of algorithms have been developed by adapting the traditional cross-entropy learning objective \cite{Cui2019,Kang2020,Cao2019,Menon2021}.
Recent studies, however, have unveiled that supervised contrastive learning (SCL) \cite{Khosla2020} may serve as a more suitable optimization target in terms of resilience against long-tail distribution \cite{Li2022a,Zhu2022}.
More precisely, SCL deliberately integrates label information into the formulation of positive and negative pairs for the contrastive loss function.
Unlike self-supervised learning, which generates positive samples through data augmentation of the anchor, SCL constructs positive samples from the same class as the anchor.
Notably, initial exploration of this approach has already yielded performance surpassing most competitive algorithms designed for long-tail distribution \cite{Cui2021,Wang2021,Kang2020a}.

Despite its merits, SCL still suffers from an inherent limitation. 
To guarantee performance, SCL necessitates a considerable batch size for generating sufficient contrastive pairs \cite{Khosla2020}, resulting in a substantial computational and memory overhead.
Notably, this issue becomes more pronounced with long-tailed data in real-world settings, where tail classes are infrequently sampled within a mini-batch or memory bank.
Consequently, the loss function's gradient is predominantly influenced by head classes, leading to a lack of information from tail classes and an inherent tendency for the model to concentrate on head classes while disregarding tail classes \cite{Zhu2022,Cui2021}. 
As an example, in the Imagenet-LT dataset, a typical batch size of $4096$ and memory size of $8192$ yield an average of fewer than one sample per mini-batch or memory bank for $212$ and $89$ classes, respectively.

In this study, we address the aforementioned issue with a simple but effective solution.
Our primary insight involves considering the sampling of an infinite number of contrastive pairs from the actual data distribution and solving the expected loss to determine the optimization objective.
By directly estimating and minimizing the expectation, the need for maintaining a large batch size is obviated.
Moreover, since all classes are theoretically equivalent in the expectation, the long-tail distribution problem in real-world data is naturally alleviated.

However, implementing our idea is not straightforward due to two obstacles: 1) the methodology for modeling the actual data distribution is typically complex, e.g., training deep generative models \cite{Goodfellow2020,SohlDickstein2015,Guo2022}; and 2) calculating the expected training loss in a closed form is difficult.
In this paper, we simultaneously tackle both challenges by proposing a novel probabilistic contrastive learning algorithm as illustrated in \cref{fig:fig1}.
Our method is inspired by the intriguing observation that deep features generally contain rich semantic information, enabling their statistics to represent the intra-class and inter-class variations of the data \cite{Wang2021a,Li2021,Cai2020}.
These methods model unconstrained features with a normal distribution from the perspective of data augmentation and obtain an upper bound of the expected cross-entropy loss for optimization.
However, due to the normalization of features in contrastive learning, direct modeling with a normal distribution is not feasible.
In addition, it is not possible to estimate the distributions of all classes in a small batch for long-tailed data.
Therefore, we adopt a reasonable and simple von Mises-Fisher distribution on the unit sphere in $\mathbb{R}^n$ to model the feature distribution, which is commonly considered as an extension of the normal distribution to the hypersphere.
It brings two advantages: 1) the distribution parameters can be estimated by maximum likelihood estimation using only the first sample moment, which can be efficiently computed across different batches during the training process;
and 2) building upon this formulation, we theoretically demonstrate that a closed form of the expected loss can be rigorously derived as the sampling number approaches infinity rather than an upper bound, which we designate as the \name~loss.
This enables us to circumvent the necessity of explicitly sampling numerous contrastive pairs and instead minimize a surrogate loss function, which can be efficiently optimized and does not introduce any extra overhead during inference.

\begin{figure}[t]
    \centering
    \includegraphics[width=0.95\linewidth]{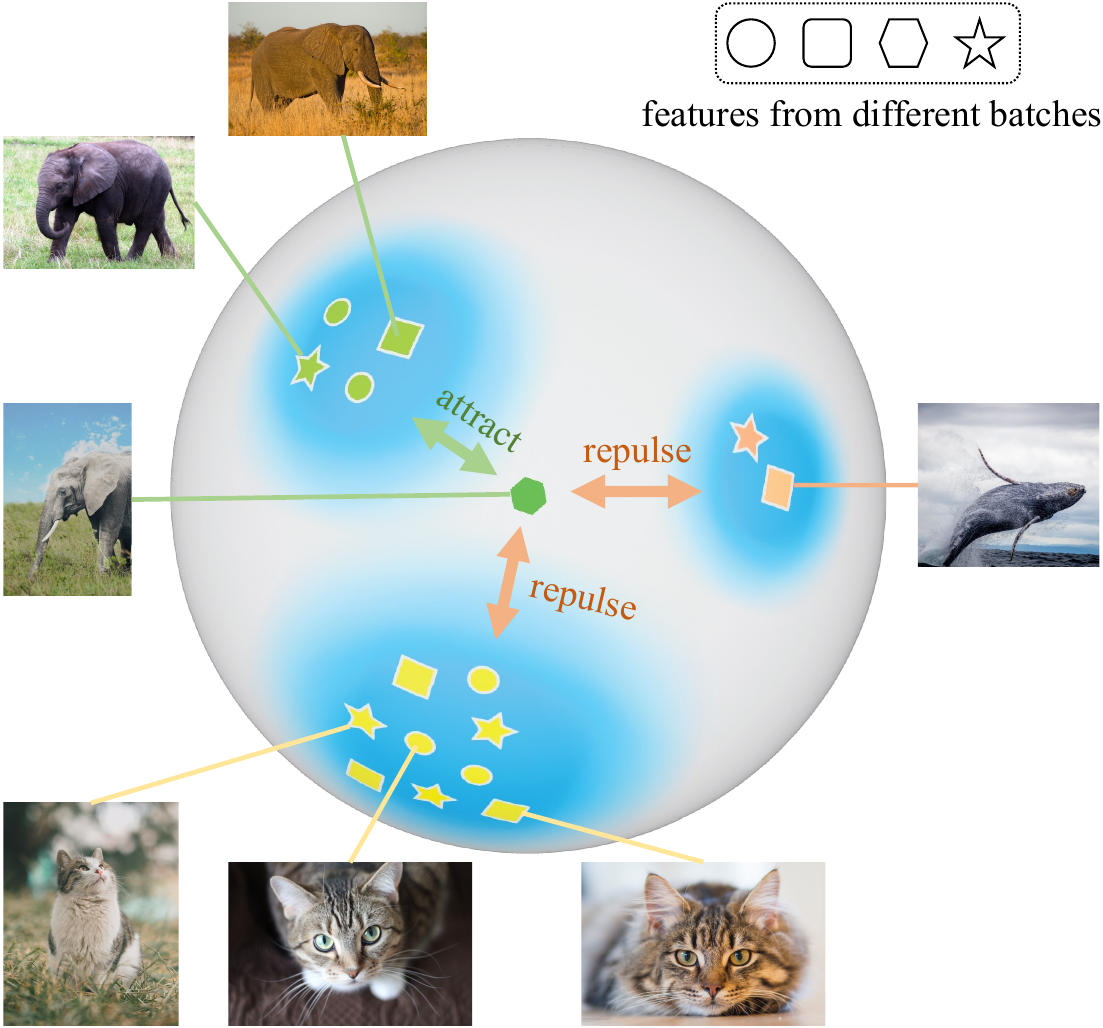}
    \caption{Illustration of Probabilistic Contrastive Learning. \name~estimates the distribution of samples based on the features from different batches and samples contrastive pairs from it. Moreover, a closed form of expected contrastive loss is derived by sampling an infinite number of contrastive pairs, which eliminates the inherent limitation of SCL on large batch sizes. }
    \label{fig:fig1}
\end{figure}

Furthermore, we extend the application of the proposed \name~algorithm to more realistic imbalanced semi-supervised learning scenarios, where only a small portion of the entire training data set possesses labels \cite{Jia2021,Miyato2019,Kingma2014,Rasmus2015}.
Semi-supervised algorithms typically employ a strategy that generates pseudo-labels for unlabeled data based on the model's predictions, using these labels to regularize the model's training process.
Consequently, the \name~algorithm can be directly applied to unlabeled samples by generating pseudo-labels grounded in the \name~loss, which can subsequently be used to estimate the feature distribution inversely.

Despite its simplicity, the proposed \name~algorithm demonstrates consistent effectiveness.
We perform extensive empirical evaluations on supervised/semi-supervised image classification  and object detection tasks with CIFAR-10/100-LT,  ImageNet-LT,
iNaturalist 2018, and LVIS v1.
The results demonstrate that \name~consistently improves the generalization performance of existing competitive long-tailed recognition methods.
Furthermore, since \name~is independent of imbalanced class distribution theoretically, experiments are also conducted on balanced datasets. The results indicate that \name~achieves enhanced performance on balanced datasets as well.

The primary contributions of this study are outlined as follows:
\begin{itemize}
    \item We propose a novel probabilistic contrastive learning (\name) algorithm for long-tailed recognition problems.
        By adopting a reasonable and simple von Mises-Fisher (vMF) distribution to model the feature distribution, we can estimate the parameters across different batches efficiently.
        \name~eliminates the inherent limitation of SCL on large batch size by sampling contrastive pairs from the estimated distribution, particularly when dealing with imbalanced data (see \cref{fig:representation branch}).
    \item We derive a closed form of expected supervised contrastive loss based on the estimated vMF distribution when the sampling number tends to infinity
        and theoretically analyze the error bound of the \name~loss (see \cref{sec:error_analysis}).
        This approach eliminates the requirement of explicitly sampling numerous contrastive pairs and, instead, focuses on minimizing a surrogate loss function. The surrogate loss function can be efficiently optimized without introducing any additional overhead during inference.
    \item We employ the proposed \name~algorithm to address imbalanced semi-supervised learning scenarios in a more realistic manner. This involves generating pseudo-labels based on the \name~loss. Subsequently, we conduct comprehensive experiments on various classification datasets to showcase the efficacy of the ProCo algorithm.

\end{itemize}

%% file: relate.tex
\section{Related Work}
\textbf{Long-tailed Recognition.}
To address the long-tailed recognition problem, early rebalancing methods can be classified into two categories: re-sampling and re-weighting.
Re-sampling techniques aid in the acquisition of knowledge pertaining to tail classes by adjusting the imbalanced distribution of training data through either undersampling~\cite{Kubat1997,Wallace2011} or oversampling~\cite{Chawla2002}.
Re-weighting methods adapt the loss function to promote greater gradient contribution for tail classes~\cite{Huang2016,Menon2013} and even samples~\cite{Cui2019}.
Nevertheless, Kang et al.~\cite{Kang2020} demonstrated that strong long-tailed recognition performance can be attained by merely modifying the classifier, without rebalancing techniques.
Furthermore, post-hoc normalisation of the classifier weights~\cite{Kang2020,Kim2020a,Zhang2019} and loss margin modification~\cite{Menon2021,Ren2020,Tan2020,Cao2019} have been two effective and prevalent methods.
Post-hoc normalisation is motivated by the observation that the classifier weight norm tends to correlate with the class distribution, which can be corrected by normalising the weights.
Loss margin modification methods incorporate prior information of class distribution into the loss function by adjusting the classifier’s margin.
Logit Adjustment~\cite{Menon2021} and Balanced Softmax~\cite{Ren2020} deduce that the classifier’s decision boundary for each class corresponds to the $\log$ of the prior probability in the training data from the probabilistic perspective, which is demonstrated to be a straightforward and effective technique.
Moreover, another common technique involves augmenting the minority classes by data augmentation techniques~\cite{He2016,Huang2017,Wang2021a,Chu2020,Zang2021,Li2021,Wang2021c,Wang2023,Huang2022a}.
MetaSAug~\cite{Li2021} employs meta-learning to estimate the variance of the feature distribution for each class and utilizes it as a semantic direction for augmenting single sample, which is inspired by implicit semantic data augmentation (ISDA)~\cite{Wang2021a}.
The ISDA employs a normal distribution to model unconstrained features for data augmentation and obtains an upper bound of the expected cross-entropy loss, which is related to our method.
Nevertheless, features' normalization makes direct modeling infeasible in contrastive learning with a normal distribution.
Hence, we adopt a mixture of von Mises-Fisher distributions on the unit sphere, allowing us to derive a closed-form of the expected contrastive loss rather than an upper bound.
The vMF distribution~\cite{Mardia2000} is a fundamental probability distribution on the unit hyper-sphere $\mathcal{S}^{p-1}$ in $\mathbb{R}^p$, which has been successfully used in deep metric learning~\cite{Zhe2019,Roth2022}, supervised learning~\cite{Scott2021,Li2021b}, and unsupervised learning~\cite{Banerjee2005}.
A recent study~\cite{wang2022towards} introduces a classifier that utilizes the von Mises-Fisher distribution to address long-tailed recognition problems. Although this approach exhibits similarities to our method in terms of employing the vMF distribution, it specifically emphasizes the quality of representation for classifiers and features, considering the distribution overlap coefficient.

\textbf{Contrastive Learning for Long-tailed Recognition.}
Recently, researchers have employed the contrastive learning to tackle the challenge of long-tailed recognition.
Contrastive learning is a self-supervised learning approach that leverages contrastive loss function to learn a more discriminative representation of the data by maximizing the similarity between positive and negative samples~\cite{Ting2020,JeanBastien2020,Tian2019,Kaiming2019}.
Khosla et al.~\cite{Khosla2020} extended the contrastive learning to the supervised contrastive learning (SCL) paradigm by incorporating label information.
However, due to the imbalance of positive and negative samples, contrastive learning also faces the problem that the model over-focusing on head categories in long-tailed recognition~\cite{Kang2020a,Fang2021,Cui2021}.
To balance the feature space, KCL~\cite{Kang2020a} uses the same number positive pairs for all the classes.
Recent studies~\cite{Li2022a,Cui2021,Zhu2022,Wang2021,Samuel2021,Cui2023} have proposed to introduce class complement for constructing positive and negative pairs.
These approaches ensure that all classes appear in every training iteration to re-balance the distribution of contrast samples.
A comprehensive comparison is provided in~\cref{sec:connection}.

Furthermore, recent advancements in multi-modal foundation models based on contrastive learning, such as CLIP~\cite{Radford2021}, have demonstrated remarkable generalization capabilities across various downstream tasks.
Inspired by this, researchers have begun to incorporate multi-modal foundational models into long-tail recognition tasks.
VL-LTR~\cite{Tian2022} develops a class-level visual-linguistic pre-training approach to associate images and textual descriptions at the class level and introduces a language-guided recognition head, effectively leveraging visual-linguistic representations for enhanced visual recognition.

\textbf{Knowledge Distillation for Long-tailed Recognition.}
Knowledge distillation involves training a student model using the outputs of a well-trained teacher model~\cite{Hinton2015}.
This approach has been increasingly applied to long-tailed learning.
For instance, LFME~\cite{Xiang2020} trains multiple experts on various, less imbalanced sample subsets (e.g., head, middle, and tail sets), subsequently distilling these experts into a unified student model.
In a similar vein, RIDE~\cite{Wang2020} introduces a knowledge distillation method to streamline the multi-expert model by developing a student network with fewer experts. 
Differing from the multi-expert paradigm, DiVE~\cite{He2021} demonstrates the efficacy of using a class-balanced model as the teacher for enhancing long-tailed learning. 
NCL~\cite{Li2022} incorporates two main components: Nested Individual Learning and Nested Balanced Online Distillation. NIL focuses on the individual supervised learning for each expert, while NBOD facilitates knowledge transfer among multiple experts.
Lastly, xERM~\cite{Zhu2022a} aims to develop an unbiased, test-agnostic model for long-tailed classification.
Grounded in causal theory, xERM seeks to mitigate bias by minimizing cross-domain empirical risk.

\textbf{Imbalanced Semi-supervised Learning (SSL).}
Semi-supervised learning is a subfield of machine learning that addresses scenarios where labeled training samples are limited, but an extensive amount of unlabeled data is available~\cite{Jia2021,Miyato2019,Kingma2014,Rasmus2015,Huang2022}.
This scenario is directly relevant to a multitude of practical problems where it is relatively expensive to produce labeled data.
The main approach in SSL is leveraging labeled data to generate pseudo-labels for unlabeled data, and then train the model with both pseudo-labeled and labeled data ~\cite{Sohn2020}.
In addition, consistency regularization or cluster assumption can be combined to further constrain the distribution of unlabeled data.
For long-tailed datasets, due to class imbalance, SSL methods will be biased towards head classes when generating pseudo-labels for unlabeled data.
Recently, researchers have proposed some methods to address the problem of pseudo-label generation in imbalanced SSL.
DARP~\cite{Kim2020} is proposed to softly refine the pseudo-labels generated from a biased model by formulating a convex optimization problem.
CReST~\cite{Wei2021} adopts an iterative approach to retrain the model by continually incorporating pseudo-labeled samples.
DASO~\cite{Oh2022} focuses on the unknown distribution of the unlabeled data and blends the linear and semantic pseudo-labels in different proportions for each class to reduce the overall bias.

%% file: method.tex
\section{Method}
\subsection{Preliminaries}
In this subsection, we start by presenting the preliminaries, laying the basis for introducing our method. Consider a standard image recognition problem. Given the training set $\mathcal{D} = \{{(\bm{x}_i, y_i)}_{i=1}^N\}$, the model is trained to map the images from the space $\mathcal{X}$ into the classes from the space $\mathcal{Y} = \{1,2,\dots,K\}$. Typically, the mapping function $\varphi$ is modeled as a neural network, which consists of a backbone feature extractor $F\colon \mathcal{X}\rightarrow \mathcal{Z}$ and a linear classifier $G\colon \mathcal{Z}\rightarrow \mathcal{Y}$.

\textbf{Logit Adjustment}~\cite{Menon2021} is a loss margin modification method. It adopts the prior probability of each class as the margin during the training and inference process. The logit adjustment method is defined as:

\begin{equation}
    \label{eq:LA}
    \mathcal{L}_{\text{LA}}(\bm{x}_i,y_i)=-\log\frac{\pi_{y_i} e^{ \varphi_{y_i}(\bm{x}_i)}}{\sum\limits_{y' \in \mathcal{Y}} \pi_{y'} e^{\varphi_{y'}(\bm{x}_i)}},
\end{equation}
where $\pi_y$ is the class frequency in the training or test set, and $\varphi_{y}$ is the logits of the class $y$.

\textbf{Supervised Contrastive Learning (SCL)}~\cite{Khosla2020} is a generalization of the unsupervised contrastive learning method.
SCL is designed to learn a feature extractor $F$ that can distinguish between positive pairs $(\bm{x}_i, \bm{x}_j)$ with the same label $y_i = y_j$ and negative pairs $(\bm{x}_i, \bm{x}_j)$ with different labels $y_i \neq y_j$.
Given any batch of sample-label pairs $B=\{(\bm{x}_i,y_i)_{i=1}^{N^B}\}$ and a temperature parameter $\tau$, two typical ways to define the SCL loss are~\cite{Khosla2020}:
\begin{align}
    {\mathcal{L}}^{\text{sup}}_{\text{out}}(\bm{x}_i,y_i) &= \frac{-1}{N^B_{y_i}}\!\sum_{p\in A(y_i)}\!\!\log{\frac{e^{\bm{z}_i\cdot\bm{z}_p/\tau}}{\sum\limits_{j = 1}^K \sum\limits_{a\in A(j)}\!\!e^{\bm{z}_i\cdot\bm{z}_a/\tau}}}, \label{eq:SCL_out} \\
    {\mathcal{L}}^{\text{sup}}_{\text{in}}(\bm{x}_i,y_i) &= - \log\Bigg\{{\frac{1}{N^B_{y_i}}\sum_{p\in A(y_i)}\frac{e^{\bm{z}_i\cdot\bm{z}_p/\tau}}{\sum\limits_{j = 1}^K \sum\limits_{a\in A(j)}e^{\bm{z}_i\cdot\bm{z}_a/\tau}}}\Bigg\}, \label{eq:SCL_in}
\end{align}
where $A(j)$ is the set of indices of the instances in the batch $B\setminus\{({x}_i, y_i)\}$ with the same label $j$, $N^{B}_{y_i}=\vert A(y_i)\vert$ is its cardinality, and $\bm{z}$ denotes the normalized features of $\bm{x}$ extracted by $F$:
\begin{equation*}
    \bm{z}_i = \frac{F(\bm{x}_i)}{\|F(\bm{x}_i)\|}, \quad \bm{z}_p = \frac{F(\bm{x}_p)}{\|F(\bm{x}_p)\|}, \quad \bm{z}_a = \frac{F(\bm{x}_a)}{\|F(\bm{x}_a)\|}.
\end{equation*}
In addition, ${\mathcal{L}}^{\text{sup}}_{\text{out}}$ and $ {\mathcal{L}}^{\text{sup}}_{\text{in}}$ denote the sum over the positive pairs relative to the location of the $\log$.
As demonstrated in~\cite{Khosla2020}, the two loss formulations are not equivalent and Jensen's inequality~\cite{Jensen1906} implies that $\mathcal{L}^{\text{sup}}_{\text{in}} \leq  \mathcal{L}^{\text{sup}}_{\text{out}}$.
Therefore, SCL adopts the latter as the loss function, since it is an upper bound of the former.

\subsection{Probabilistic Contrastive Learning}
\label{sec:probabilistic_contrastive_learning}

As aforementioned, for any example in a batch, SCL considers other examples with the same labels as positive samples, while the rest are viewed as negative samples.
Consequently, it is essential for the batch to contain an adequate amount of data to ensure each example receives appropriate supervision signals.
Nevertheless, this requirement is inefficient as a larger batch size often leads to significant computational and memory burdens.
Furthermore, in practical machine learning scenarios, the data distribution typically exhibits a long-tail pattern, with infrequent sampling of the tail classes within the mini-batches.
This particular characteristic necessitates further enlargement of the batches to effectively supervise the tail classes.

To address this issue, we propose a novel probabilistic contrastive (\name) learning algorithm that estimates the feature distribution and samples from it to construct contrastive pairs. Our method is inspired by~\cite{Li2021,Wang2021a},  which employ normal distribution to model unconstrained features from the perspective of data augmentation and obtain an upper bound of the expected loss for optimization. However, the features in contrastive learning are constrained to the unit hypersphere, which is not suitable for directly modeling them with a normal distribution.
Moreover, due to the imbalanced distribution of training data, it is infeasible to estimate the distribution parameters of all classes in a small batch.
Therefore, we introduce a simple and reasonable von Mises-Fisher distribution defined on the hypersphere, whose parameters can be efficiently estimated by maximum likelihood estimation across different batches. Furthermore, we rigorously derive a closed form of expected SupCon loss rather than an upper bound for efficient optimization and apply it to semi-supervised learning.

\textbf{Distribution Assumption.}
As previously mentioned, the features in contrastive learning are constrained to be on the unit hypersphere.
Therefore, we assume that the features follow a mixture of von Mises–Fisher (vMF) distributions~\cite{Mardia2000}, which is often regarded as a generalization of the normal distribution to the hypersphere.
The probability density function of the vMF distribution for the random $p$-dimensional unit vector $\bm{z}$ is given by:
\begin{align}
  \label{eq:vMF}
f_{p}(\bm{z} ;{\bm {\mu }},\kappa ) &= \frac{1}{C_{p}(\kappa )}e^{{\kappa {\bm {\mu }}^{\top} \bm {z} }}, \\
  \label{eq:vMF_const}
 C_{p}(\kappa ) &= {\frac {(2\pi )^{p/2}I_{(p/2-1)}(\kappa )}{\kappa ^{p/2-1}}},
\end{align}
where $\bm z$ is a $p$-dimensional unit vector, ${\kappa \geq 0,\left\Vert {\bm {\mu }}\right\Vert_2 =1}$ and 
$I_{(p/2-1)}$ denotes the modified Bessel function of the first kind at order $p/2-1$, which is defined as:
\begin{equation}
    \label{eq:Bessel}
    I_{(p/2-1)}(z) = \sum_{k=0}^{\infty }{\frac{1}{k!\Gamma (p/2-1+k+1)}} (\frac{z}{2})^{2k+p/2-1}.
\end{equation}
The parameters $\bm {\mu} $ and $\kappa$ are referred to as the mean direction and concentration parameter, respectively.
A higher concentration around the mean direction $\bm{\mu}$ is observed with greater $\kappa$, and the distribution becomes uniform on the sphere when $\kappa=0$.

\textbf{Parameter Estimation.}
Under the above assumption, we employ a mixture of vMF distributions to model the feature distribution:
\begin{equation}
    P(\bm{z}) = \sum_{{y}=1}^{K} P(y) P(\bm{z}|y) = \sum_{{y}=1}^{K} \pi_{y} \frac{\kappa_{y}^{p/2-1}e^{\kappa_y\bm{\mu}_{y}^\top\bm{z}}}{(2\pi)^{p/2} I_{p/2-1}(\kappa_y)} \label{eq:vmf},
\end{equation}
where the probability of a class $y$ is estimated as $\pi_y$, which corresponds to the frequency of class $y$ in the training set.
The mean vector $\bm{\mu}_y$ and concentration parameter $\kappa_y$ of the feature distribution are estimated by maximum likelihood estimation.

Suppose that a series of $N$ independent unit vectors $\{(\bm{z}_i)_i^N\}$ on the unit hyper-sphere $\mathcal{S}^{p-1}$ are drawn from a vMF distribution of class $y$.
The maximum likelihood estimates of the mean direction ${\bm{\mu}_y}$ and concentration parameter $\kappa_y$ satisfy the following equations:
\begin{align}
    \bm{\mu}_y &={\bar{\bm{z}}}/{\bar {R}}, \label{eq:mle_mu} \\
    A_p(\kappa_y) &= \frac{I_{p/2}(\kappa_y)}{I_{p/2-1}(\kappa_y)} = \bar{R}, \label{eq:mle_kappa}
\end{align}
where $\bar {\bm{z}}={\frac {1}{N}}\sum _{i}^{N}\bm{z}_{i}$ is the sample mean and  $\bar {R}=\|{\bar {\bm{z}}}\|_2$ is the length of sample mean.
A simple approximation~\cite{Sra2012} to $\kappa_y$ is:
\begin{equation}
    \label{eq:mle_approx_kappa}
    {\hat {\kappa}_y}={\frac {{\bar {R}}(p-{\bar {R}}^{2})}{1-{\bar {R}}^{2}}}. 
\end{equation}

Furthermore, the sample mean of each class is estimated in an online manner by aggregating statistics from the previous mini-batch and the current mini-batch.
Specifically, we adopt the estimated sample mean of the previous epoch for maximum likelihood estimation, while maintaining a new sample mean from zero initialization in the current epoch through the following online estimation algorithm:
\begin{equation}
    \label{eq:online_mean}
\bar{\bm{{z}}}_j^{(t)} = \frac{n_j^{(t-1)}\bar{\bm{{z}}}_j^{(t-1)} + m_j^{(t)} \bar{\bm{z}}^{\prime(t)}_j} {n_j^{(t-1)} +m_j^{(t)}},
\end{equation}
where ${\bar{\bm{z}}_j}^{(t)}$ is the estimated sample mean of class $j$ at step $t$ and $\bar{\bm{z}}^{\prime(t)}_j$ is the sample mean of class $j$ in current mini-batch.
$n_j^{(t-1)}$ and $m_j^{(t)}$ represent the numbers of samples in the previous mini-batches and the current mini-batch, respectively.

\textbf{Loss Derivation.}
Built upon the estimated parameters, a straightforward approach may be sampling contrastive pairs from the mixture of vMF distributions.
However, we note that sampling sufficient data from the vMF distributions at each training iteration is inefficient. To this end, we leverage mathematical analysis to extend the number of samples to infinity and rigorously derive a closed-form expression for the expected contrastive loss function.

\begin{proposition}
    \label{prop:sample}
    Suppose that the parameters of the mixture of vMF distributions are $\pi_y$, $\bm{\mu}_y$, and $\kappa_y, y=1,\cdots,K$ and let the sampling number $N \to \infty$.
    Then we have the expected contrastive loss function, which is given by:

    \begin{align}
        {\mathcal{L}}_{\rm {out}}(\bm{z_i},y_i) &= \frac{-\bm{z}_i\cdot A_p(\kappa_{y_i}) \bm{\mu}_{y_i}}{\tau} + \log \bigg( {\sum\limits_{j = 1}^K \pi_j \frac{C_p(\tilde \kappa_j)}{C_p(\kappa_j)} } \bigg),  \label{eq:proco_out} \\
        {\mathcal{L}}_{\rm {in}}(\bm{z_i},y_i) &= -\log \bigg( {\pi_{y_i} \frac{C_p(\tilde \kappa_{y_i})}{C_p(\kappa_{y_i})} } \bigg) + \log \bigg( {\sum\limits_{j = 1}^K \pi_j \frac{C_p(\tilde \kappa_j)}{C_p(\kappa_j)} } \bigg), \label{eq:proco_in}
    \end{align}
    where $\tilde{\bm{z}}_j \sim {\rm vMF}(\bm{\mu}_{j}, {\kappa}_{j})$,
$\tilde \kappa_j = ||\kappa_{j} \bm{\mu}_j +  \bm{z}_i/\tau  ||_2$, $\tau$ is the temperature parameter.

\end{proposition}

\begin{proof}

According to the definition of supervised contrastive loss in~\cref{eq:SCL_out}, we have

\begin{align}
 {\mathcal{L}}^{\text{sup}}_{\text{out}} &= \frac{-1}{N_{y_i}}\sum_{p\in A(y_i)} {\bm{z}_i\cdot\bm{z}_p/\tau} \notag \\
                                         &+ \log \bigg( {\sum\limits_{j = 1}^K N \frac{N_j}{N} \frac{1}{N_j} \sum\limits_{a\in A(j)}e^{\bm{z}_i\cdot\bm{z}_a/\tau}}\bigg),
\end{align}
where $N_j$ is the sampling number of class $j$ and satisfies $\lim_{N \to \infty} N_j/N = \pi_j$.

Let $N \to \infty$ and omit the constant term $\log N$, we have the following loss function:

\begin{align}
{\mathcal{L}}_{\text{out}} &= \frac{-\bm{z}_i\cdot \mathbb{E}[\tilde{\bm{z}}_{y_i}]}{\tau} + \log \bigg( {\sum\limits_{j = 1}^K \pi_j \mathbb{E}[ e^{\bm{z}_i\cdot \tilde{\bm{z}}_j/\tau}}] \bigg) \\
                                              &= \frac{-\bm{z}_i\cdot A_p(\kappa_{y_i}) \bm{\mu}_{y_i}}{\tau} + \log \bigg( {\sum\limits_{j = 1}^K \pi_j \frac{C_p(\tilde \kappa_j)}{C_p(\kappa_j)} } \bigg). \label{eq:proco_out_proof}
\end{align}

\cref{eq:proco_out_proof} is obtained by leveraging the expectation and moment-generating function of vMF distribution:
\begin{align}
    {\mathbb {E} \left(\mathbf{z}\right)} &= A_p(\kappa) \bm \mu, A_p(\kappa) = \frac{I_{p/2}(\kappa)}{I_{p/2-1}(\kappa)}, \label{eq:expectation} \\
    \mathbb {E} \left(e^{\mathbf {t} ^{\mathrm {T} }\mathbf {z} }\right) &= \frac{C_p(\tilde \kappa)}{C_p(\kappa)} , \tilde \kappa = ||\kappa \bm \mu + \mathbf{t} ||_2. \label{eq:moment}
\end{align}

Similar to~\cref{eq:proco_out}, we can obtain the other loss function from~\cref{eq:SCL_in} as follows:

\begin{equation}
    {\mathcal{L}}_{\text{in}} = -\log \bigg( {\pi_{y_i} \frac{C_p(\tilde \kappa_{y_i})}{C_p(\kappa_{y_i})} } \bigg) + \log \bigg( {\sum\limits_{j = 1}^K \pi_j \frac{C_p(\tilde \kappa_j)}{C_p(\kappa_j)} } \bigg). \label{eq:proco_in_proof}
\end{equation}

\end{proof}

Based on the above derivation, we obtain the expected formulations for the two SupCon loss functions.
Since $\mathcal{L}_{\text{in}}$ enforces the margin modification as shown in~\cref{eq:LA}, we adopt it as the surrogate loss:
\begin{equation}
    {\mathcal{L}}_{\text{\name}} = -\log \left( {\pi_{y_i} \frac{C_p(\tilde \kappa_{y_i})}{C_p(\kappa_{y_i})} } \right) + \log \left( {\sum\limits_{j = 1}^K \pi_j \frac{C_p(\tilde \kappa_j)}{C_p(\kappa_j)} } \right). \label{eq:proco}
\end{equation}
The empirical comparison of $\mathcal{L}_{\text{in}}$ and $\mathcal{L}_{\text{out}}$ is shown in~\cref{tab:comparison of loss}.

Instead of costly sampling operations, we implicitly achieve infinite contrastive samples through the surrogate loss and can optimize it in a much more efficient manner.
This design elegantly addresses the inherent limitations of the SCL, i.e., relying on large batch sizes (see~\cref{fig:representation branch}).
Furthermore, the assumption of feature distribution and the estimation of parameters can effectively capture the diversity of features among different classes, which enables our method to achieve stronger performance even without the sample-wise contrast as SCL (see~\cref{tab:ablation study}).

\textbf{Numerical Computation.}
Due to PyTorch only providing the GPU implementation of the zeroth and first-order modified Bessel functions, one approach for efficiently computing the high-order function in \name~is using the following recurrence relation:

\begin{equation}
I_{\nu+1}(\kappa) = \frac{2\nu}{\kappa} I_{\nu}(\kappa) - I_{\nu-1}(\kappa).
\end{equation}
However, this method exhibits numerical instability when the value of $\kappa$ is not sufficiently large.

Hence, we employ the Miller recurrence algorithm~\cite{abramowitz1964handbook}.
To compute $I_{p/2-1}(\kappa)$ in \name, we follow these steps: First, we assign the trial values $1$ and $0$ to $I_M(\kappa)$ and $I_{M+1}(\kappa)$, respectively. Here, $M$ is a chosen large positive integer, and in our experiments, we set $M=p$.
Then, using the inverse recurrence relation:
\begin{equation}
I_{\nu-1}(\kappa) = \frac{2\nu}{\kappa} I_{\nu}(\kappa) + I_{\nu+1}(\kappa),
\end{equation}
we can compute $I_{\nu}(\kappa)$ for $\nu = M-1, M-2, \cdots, 0$.
The value of $I_{p/2-1}(\kappa)$ obtained from this process is denoted as $\tilde{I}_{p/2-1}(\kappa)$, and $I_0(\kappa)$ is denoted as $\tilde{I}_0(\kappa)$.
Finally, we can then compute $I_{p/2-1}(\kappa)$ as follows:
\begin{equation}
I_{p/2-1}(\kappa) = \frac{I_0(\kappa)}{\tilde{I}_0(\kappa)} \tilde{I}_{p/2-1}(\kappa).
\end{equation}

\textbf{Overall Objective.}
Following the common practice in long-tailed recognition~\cite{Wang2021,Cui2021,Zhu2022}, we adopt a two-branch design.
The model consists of a classification branch based on a linear classifier $G(\cdot)$ and a representation branch based on a projection head $P(\cdot)$, which is an MLP that maps the representation to another feature space for decoupling with the classifier.
Besides, a backbone network $F(\cdot)$ is shared by the two branches.
For the classification branch and the representation branch, we adopt the simple and effective logit adjustment loss $\mathcal{L}_{\text{LA}}$ and our proposed loss $\mathcal{L}_{\text{\name}}$ respectively.
Finally, the loss functions of the two branches  are weighted and summed up as the overall loss function:
\begin{equation}
    \mathcal{L} = \mathcal{L}_{\text{LA}} + \alpha \mathcal{L}_{\text{\name}}, \label{eq:loss}
\end{equation}
where $\alpha$ is the weight of the representation branch.

In general, by introducing an additional feature branch during training, our method can be efficiently optimized with stochastic gradient descent (SGD) algorithm along with the classification branch and does not introduce any extra overhead during inference.

\textbf{Compatibility with Existing Methods.}
In particular, our approach is appealing in that it is a general and flexible framework. It can be easily combined with existing works applied to the classification branch, such as different loss functions, multi-expert framework, etc (see \cref{tab:imagenet-lt-resnet50}).

The pseudo-code of our algorithm is shown in~\cref{alg}.
\begin{center}
        \begin{algorithm}[ht]
            \caption{The \name~Algorithm.}
            \small
            \label{alg}
        \begin{algorithmic}[1]
            \STATE {\bfseries Input:} Training set $\mathcal{D}$, loss weight $\alpha$
            \STATE Randomly initialize
            the parameters $\bm{\Theta}$ of backbone $F$, projection head $P$ and classifier $G$
            \FOR{$t=0$ {\bfseries to} $T$}
            \STATE Sample a mini-batch $\{ \bm{x}_i, {y_i} \}_{i=1}^B$  from $\mathcal{D}$
            \STATE Compute $\bm{z}_{i} = \frac{P(F(\bm{x}_{i}))}{\|P(F(\bm{x}_{i}))\|} $ and $G(F(\bm{x}_{i}))$
            \STATE Estimate $\bm{\mu}$ and $\kappa$ according to~\cref{eq:mle_mu} and~\cref{eq:mle_approx_kappa}
            \STATE Compute ${\mathcal{L}}$
            according to~\cref{eq:loss}
            \STATE Update $\bm{\Theta}$ with SGD
            \ENDFOR
            \STATE {\bfseries Output:} $\bm{\Theta}$
        \end{algorithmic}
        \end{algorithm}
\end{center}

\subsection{Theoretical Error Analysis}
\label{sec:error_analysis}

To further explore the theoretical foundations of our approach,
we establish an upper bound on the generalization error and excess risk for the \name~loss, as defined in~\cref{eq:proco}.
For simplicity, our analysis focuses on the binary classification scenario, where the labels $y$ belong to the set $\{-1, +1\}$.

\begin{assumption}
    \label{assumption:p_tau}
    $p\tau \gg 1$, with $\tau$ representing the temperature parameter and $p$ the dimensionality of the feature space.
\end{assumption}

\begin{proposition}[Generalization Error Bound]
    \label{prop:gen_bound}
    Under the \cref{assumption:p_tau}, the following generalization bound is applicable with a probability of at least $1-\delta/2$.
    For every class $y \in \{-1,1\}$ and for estimated parameters $\hat{\bm{\mu}}$ and $\hat{\kappa}$, the bound is expressed as:
    \begin{align}
        & \mathbb{E}_{\bm{z}|y} \mathcal{L}_{\text{\name}}(y,\bm{z}; \hat{\bm{\mu}}, \hat{\kappa}) - \frac{1}{N_{y}}\sum_{i} \mathcal{L}_{\text{\name}}(y,\bm{z}_i; \hat{\bm{\mu}}, \hat{\kappa}) \notag \\
        \le &\sqrt{\frac{2}{N_y} \bm{w}^\top \Sigma_y \bm{w}  \ln \frac{2}{\delta}} + \frac{\ln (2/\delta)}{3N_{y}} \log (1 + e^{||\bm{w}||_2 - b y}).
    \end{align}
    The generalization bound across all classes, with a probability of at least $1-\delta$, is thus:
    \begin{align}
        \mathbb{E}_{(\bm{z},y)} \mathcal{L}_{\text{\name}}(y,\bm{z};\hat{\bm{\mu}}, \hat{\kappa}) \le & \sum_{y\in\{-1,1\}} \frac{P(y)}{N_y}\sum_{i} \mathcal{L}_{\text{\name}}(y,\bm{z}_i; \hat{\bm{\mu}}, \hat{\kappa}) \notag \\
        + \sum_{y\in\{-1,1\}}& \frac{P(y)\ln (2/\delta)}{3N_y} \log(1 + e^{||\bm{w}||_2 - b y}) \notag \\
        + \sum_{y\in\{-1,1\}}& P(y) \sqrt{\frac{2}{N_y} \bm{w}^\top (\Sigma_y )\bm{w}  \ln \frac{2}{\delta}},
    \end{align}
    where $N_y$ denotes the number of samples in class $y$, $\bm{w} = (\hat{\bm{\mu}}_{+1} - \hat{\bm{\mu}}_{-1})/\tau$,
    $b = \frac{1}{2\tau^2} (\frac{1}{\kappa_{+1}} - \frac{1}{\kappa_{-1}}) + \log \frac{\pi_{+1}}{\pi_{-1}}$,
    and $\Sigma_y$ is the covariance matrix of $\bm{z}$ conditioned on $y$. 
\end{proposition}
In our experimental setting, $\tau \approx 0.1$ and $p > 128$, thus \cref{assumption:p_tau} is reasonable in practice.
\cref{prop:gen_bound} indicates that the generalization error gap is primarily controlled by the sample size and the data distribution variance.
This finding corresponds to the insights from~\cite{Jitkrittum2022, Menon2021},
affirming that our method does not introduce extra factors in the error bound, nor does it expand the error bound.
This theoretically assures the robust generalizability of our approach.

Furthermore, our approach relies on certain assumptions regarding feature distribution and parameter estimation.
To assess the influence of these parameters on model performance, we derive an excess risk bound.
This bound measures the deviation between the expected risk using estimated parameters and the Bayesian optimal risk, which is the expected risk with parameters of the ground-truth underlying distribution.

\begin{assumption}
    \label{assumption:distribution}
    The feature distribution of each class follows a von Mises-Fisher (vMF) distribution, characterized by parameters $\bm{\mu}^\star$ and $\kappa^\star$.
\end{assumption}

\begin{proposition}[Excess Risk Bound]
    \label{prop:excess_risk}
    Given \cref{assumption:p_tau,assumption:distribution}, the following excess risk bound holds:
    \begin{align}
        & \mathbb{E}_{(\bm{z},y)} \mathcal{L}_{\text{\name}}(y,\bm{z};\hat{\bm{\mu}}, \hat{\kappa}) - \mathbb{E}_{(\bm{z},y)} \mathcal{L}_{\text{\name}}(y,\bm{z};{\bm{\mu}}^\star, \kappa^\star) \notag \\
        =& \mathcal{O}(\Delta\bm{\mu}+\Delta\frac{1}{\kappa}),
    \end{align}
    where $\Delta\bm{\mu} = \hat{\bm{\mu}} - {\bm{\mu}}^\star$, $\Delta\frac{1}{\kappa} = \frac{1}{\hat{\kappa}} - \frac{1}{\kappa^\star}$.
\end{proposition}
\cref{assumption:distribution}~is the core assumption of our method.
Building upon this, \cref{prop:excess_risk}~demonstrates that the excess risk associated with our method is primarily governed by the first-order term of the estimation error in the parameters.

\subsection{\name~for Semi-supervised Learning}

In order to further validate the effectiveness of our method, we also apply \name~to semi-supervised learning.
\name~can be directly employed by generating pseudo-labels for unlabeled data, which can subsequently be utilized to estimate the distribution inversely.
In our implementation, we demonstrate that simply adopting a straightforward approach like FixMatch~\cite{Sohn2020} to generate pseudo-labels will result in superior performance.
FixMatch's main concept lies in augmenting unlabeled data to produce two views and using the model's prediction on the weakly augmented view to generate a pseudo-label for the strongly augmented sample.
Specifically, owing to the introduction of feature distribution in our method, we can compute the \name~loss of weakly augmented view for each class to represent the posterior probability $P(y|\mathbf{z})$, thus enabling the generation of pseudo-labels.

\begin{table}[t]
\caption{Comparison with contrastive learning methods for long-tailed recognition. CC denotes class complement, CR denotes calculable representation, and MM denotes margin modification.}
\centering
\small
\begin{tabular}{c|ccc} 
\toprule[1pt]
Method     & CC & CR & MM \\
\midrule 
SCL~\cite{Khosla2020} \& BF~\cite{Hou2022} \& KCL~\cite{Kang2020a}  & \xmark & \xmark & \xmark \\
 ELM~\cite{Jitkrittum2022}                      & \xmark & \xmark & \cmark \\
TSC~\cite{Li2022a}                             & \cmark & \cmark & \xmark \\
BCL~\cite{Zhu2022} \& Hybrid-PSC~\cite{Wang2021}& \cmark & \xmark & \xmark \\
PaCo~\cite{Cui2021} \& GPaCo~\cite{Cui2023}    & \cmark & \xmark & \cmark \\
DRO-LT~\cite{Samuel2021}                       & \cmark & \cmark & \xmark \\
\midrule
\cellcolor{lightgray!50}\name                 & \cellcolor{lightgray!50}\cmark & \cellcolor{lightgray!50}\cmark & \cellcolor{lightgray!50}\cmark \\
\bottomrule[1pt]
\end{tabular}
\label{tab:comparison of scl}
\end{table}
\subsection{Connection with Related Work}
\label{sec:connection}

In the following, we discuss the connections between our method and related works on contrastive learning for long-tailed recognition.
Recent studies proposed to incorporate class complement in the construction of positive and negative pairs.
These methods ensure that all classes appear in every iteration of training to rebalance the distribution of contrast samples.
Moreover, researchers also introduced margin modification in contrastive learning due to its effectiveness.
Therefore, we mainly discuss three aspects, namely class complement, calculable representation, and margin modification.
Class complement represents introducing global representations of each class in contrastive learning for constructing positive and negative samples.
Calculable representation corresponds to the class complements are computed from the features rather than learnable parameters.
Margin modification represents adjusting the contrastive loss according to the prior frequency of different classes in the training set as shown in~\cref{eq:LA}.

In our method, we introduce the class complement based on the feature distribution by estimation from the features, which is not a learnable parameter.
If the class representation is a learnable parameter $\bm{w}$ and we ignore the contrast between samples, then we have the following contrastive loss:
\begin{equation}
    \mathcal{L}(\bm{z}_{i}, y_{i}) = -\log \frac{e^{\bm{w}_{y_{i}}^\top \bm{z}_{i}/\tau}}{\sum_{j=1}^{N} e^{\bm{w}_{j}^\top \bm{z}_{i}/\tau}},
\end{equation}
where $\bm{w}_{j}$ and $\bm{z}_{i}$ are normalized to unit length, and $\tau$ is the temperature parameter.
This is equivalent to cosine classifier (normalized linear activation)~\cite{Wang2018}---a variant of cross-entropy loss, which has been applied to long-tailed recognition algorithms~\cite{Liu2022,Wang2021b,Zhu2022}.
Therefore, the sole introduction of the learnable parameter $\bm{w}$ is analogous to the role played by the weight in the classification branch, which is further validated by the empirical results in~\cref{tab:ablation study}.

The related works are summarized in~\cref{tab:comparison of scl}.
BatchFormer (BF)~\cite{Hou2022} is proposed to equip deep neural networks with the capability to investigate the sample relationships within each mini-batch. 
This is achieved by constructing a Transformer Encoder Network among the images, thereby uncovering the interrelationships among the images in the mini-batch.
BatchFormer facilitates the propagation of gradients from each label to all images within the mini-batch, a concept that mirrors the approach used in contrastive learning.

Embedding and Logit Margins (ELM)~\cite{Jitkrittum2022} proposes to enforce both embedding and logit margins, analyzing in detail the benefit of introducing margin modification in contrastive learning.
TSC~\cite{Li2022a} introduces class representation by pre-generating and online-matching uniformly distributed targets to each class during the training process.
However, the targets do not have class semantics.
Hybrid-PSC~\cite{Wang2021}, PaCo~\cite{Cui2021}, GPaCo~\cite{Cui2023} and BCL~\cite{Zhu2022} all introduce learnable parameters as class representation.
PaCo also enforces margin modification when constructing contrastive samples.
DRO-LT~\cite{Samuel2021} computes the centroid of each class and utilizes it as the class representation, which is the most relevant work to ours.
The loss function and uncertainty radius in DRO-LT are devised by heuristic design from the metric learning and robust perspective.
Moreover, DRO-LT considers a sample and corresponding centroid as a positive pair.
But in contrast to SCL, the other samples in the batch and the centroid are treated as negative pairs, disregarding the label information of the other samples, which is somewhat pessimistic.

Compared with the above methods, \name~is derived rigorously from contrastive loss and a simple assumption on feature distribution.
Furthermore, \name~also enforces margin modification, which is a key component for long-tailed recognition.
\cref{sec:supervised} demonstrates the superiority of \name~over the above methods.

%% file: exp.tex
\section{Experiment}

In this section, we validate the effectiveness of our method on supervised/semi-supervised learning.
First, we conduct a range of analytical experiments to confirm our hypothesis and analyze each component of the method, including 1) performance of the representation branch, 2) comparison of more settings, 3) comparison between two formulations of loss, 4) sensitivity analysis of hyper-parameters, 5) data augmentation strategies.
Subsequently, we compare our method with existing supervised learning methods on long-tailed datasets such as CIFAR-10/100-LT, ImageNet-LT, and iNaturalist 2018.
Finally, experiments on balanced datasets, semi-supervised learning, and long-tailed object detection tasks are conducted to confirm the broad applicability of our method.

\subsection{Dataset and Evaluation Protocol}

We perform long-tailed image classification experiments on four prevalent long-tailed image classification datasets: CIFAR-10/100-LT, ImageNet-LT, and iNaturalist.
Following~\cite{Liu2022,Kang2020}, we partition all categories into three subsets based on the number of training samples: Many-shot ($>100$ images), Medium-shot ($20-100$ images), and Few-shot ($<20$ images).
The top-1 accuracy is reported on the respective balanced validation sets.
In addition, we conducted experiments on balanced image classification datasets and long-tailed object detection datasets to verify the broad applicability of our method.
The effectiveness of instance segmentation was assessed using the mean Average Precision (AP$_m$) for mask predictions, calculated at varying Intersection over Union (IoU) thresholds ranging from 0.5 to 0.95, and aggregated across different categories.
The AP values for rare, common, and frequent categories are represented as AP$_r$, AP$_c$, and AP$_f$, respectively, while the AP for detection boxes is denoted as AP$_b$.

\textbf{CIFAR-10/100-LT.}\quad
CIFAR-10-LT and CIFAR-100-LT are the long-tailed variants of the original CIFAR-10 and CIFAR-100~\cite{krizhevsky2009learning} datasets, which are derived by sampling the original training set.
Following~\cite{Cao2019,Cui2019}, we sample the training set of CIFAR-10 and CIFAR-100 with an exponential function $N_j = N \times \lambda^j$, where $\lambda \in (0,1)$, $N$ is the size of the original training set, and $N_j$ is the sampling quantity for the $j$-th class.
The original balanced validation sets of CIFAR-10 and CIFAR-100 are used for testing.
The imbalance factor $\gamma=\text{max}(N_j)/\text{min}(N_j)$ is defined as the ratio of the number of samples in the most and the least frequent class.
We set $\gamma$ at typical values $10,50,100$ in our experiments.

\textbf{ImageNet-LT.}\quad ImageNet-LT is proposed in~\cite{Liu2022}, which is constructed by sampling a subset of ImageNet following the Pareto distribution with power value $\alpha_p=6$.
It consists of $115.8$k images from $1000$ categories with cardinality ranging from $1280$ to $5$.

\textbf{iNaturalist 2018.}\quad iNaturalist 2018~\cite{VanHorn2018} is a severely imbalanced large-scale dataset.
It contains $8142$ classes of $437.5$k images, with an imbalance factor $\gamma=500$ with cardinality ranging from $2$ to $1000$.
In addition to long-tailed image classification, iNaturalist 2018 is also utilized in fine-grained image classification.

\textbf{CUB-200-2011.}\quad The Caltech-UCSD Birds-200-2011~\cite{Wah2011} is a prominent resource for fine-grained visual categorization tasks. Comprising 11,788 images across 200 bird subcategories, it is split into two sets: 5,994 images for training and 5,794 for testing.

\textbf{LVIS v1.}\quad The Large Vocabulary Instance Segmentation (LVIS) dataset~\cite{Gupta2019} is notable for its extensive categorization, encompassing 1,203 categories with high-quality instance mask annotations. LVIS v1 is divided into three splits: a training set with 100,000 images, a validation (val) set with 19,800 images, and a test-dev set, also with 19,800 images. Categories within the training set are classified based on their prevalence as rare (1-10 images), common (11-100 images), or frequent (over 100 images).

\begin{figure}[t]
    \centering
    \includegraphics[width=0.8\linewidth]{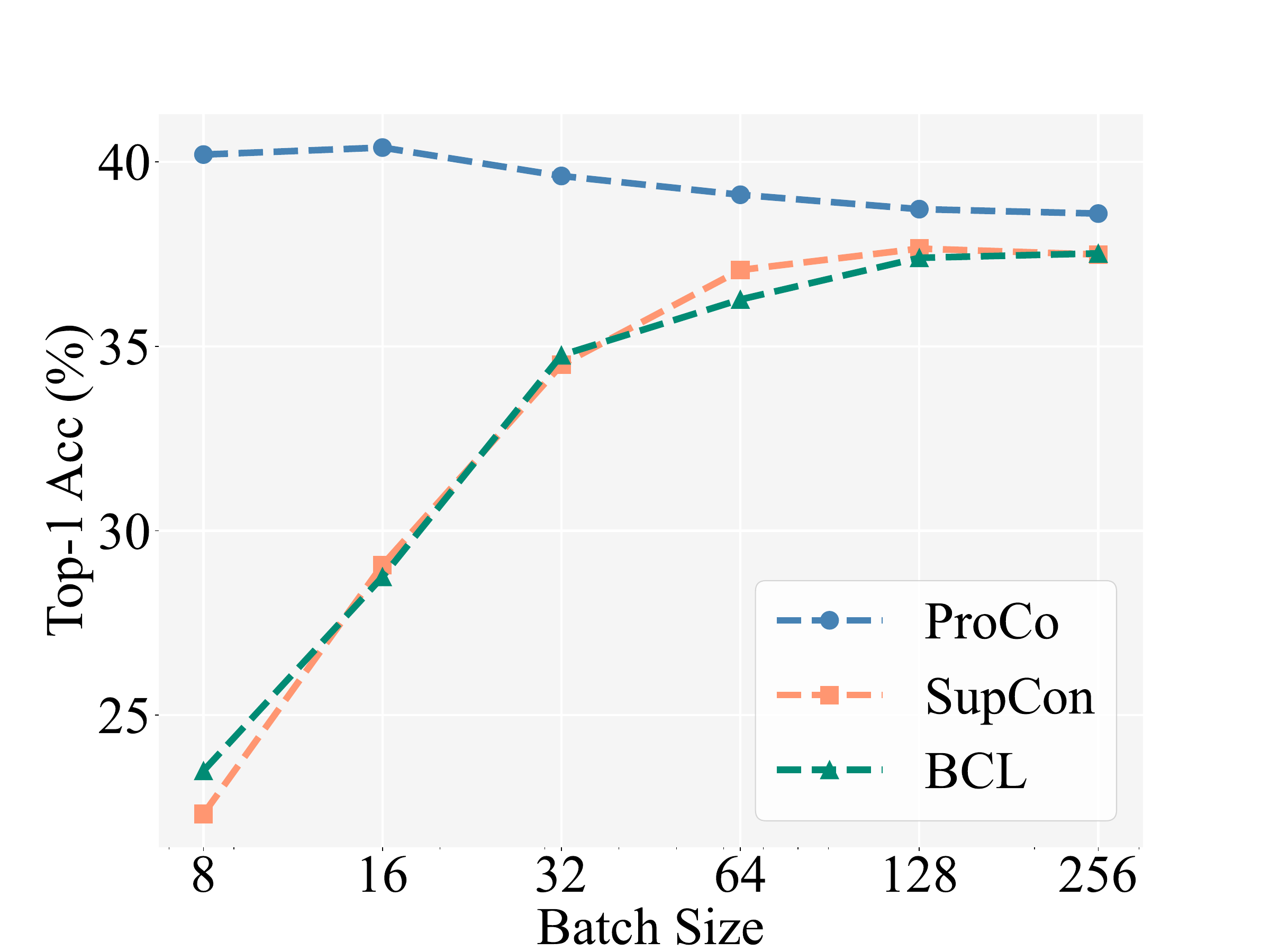}
    \caption{Performance of the representation branch. We train the model for 200 epochs.}
     \label{fig:representation branch}
\end{figure}

\begin{table}[t]
\caption{
    Comparison of different class complements. EMA denotes exponential moving average.
}

\centering
\small
\begin{tabular}{c|c}
\toprule[1pt]
Class Complement  &  Top-1 Acc.\\
\midrule
EMA Prototype   &  51.6\\
Centroid Prototype  &  52.0 \\
Normal Distribution & 52.1 \\
\midrule
\cellcolor{lightgray!50}\name          &  \cellcolor{lightgray!50}\textbf{52.8} \\
\bottomrule[1pt]
\end{tabular}
\label{tab:comparison of class representation}
\end{table}

\begin{table}[t]
    \caption{Comparison of more settings. SC denotes supervised contrastive loss, CC denotes class complement, CR denotes calculable representation, and MM denotes margin modification.
}
    \centering
    \small
    \begin{tabular}{c c c c|c}
        \toprule
         SC  &  CC  & CR & MM  &  Top-1 Acc.  \\
        \midrule
        \xmark &\xmark  & \xmark & \xmark & 50.5 \\
        \xmark &\cmark & \xmark & \xmark &  50.6 \\
        \cmark &\xmark  & \xmark & \xmark & 51.8 \\
        \cmark &\cmark & \xmark & \xmark  & 51.9 \\
        \midrule
        \cmark &\cmark  & \cmark & \cmark & 52.6\\
        \cellcolor{lightgray!50}\xmark &\cellcolor{lightgray!50}\cmark & \cellcolor{lightgray!50}\cmark & \cellcolor{lightgray!50}\cmark & \cellcolor{lightgray!50}\textbf{52.8} \\
        \bottomrule
\end{tabular}
\label{tab:ablation study}
\end{table}

\begin{table}[t]
    \caption{Comparison of different loss formulations.}
    
    \centering
    \small
    \begin{tabular}{c|ccc}
    \toprule[1pt]
    Imbalance Factor                      & 100   & 50    & 10 \\
    \midrule
    $\mathcal{L}_\text{out}$ & 52.0 & 56.6 & 65.1 \\
    \cellcolor{lightgray!50}$\mathcal{L}_\text{in}$  & \cellcolor{lightgray!50}\textbf{52.8} & \cellcolor{lightgray!50}\textbf{57.1} & \cellcolor{lightgray!50}\textbf{65.5} \\
    \bottomrule[1pt]
    \end{tabular}
    \label{tab:comparison of loss}
\end{table}

\begin{figure}[t]
    \centering
    \includegraphics[width=0.8\linewidth]{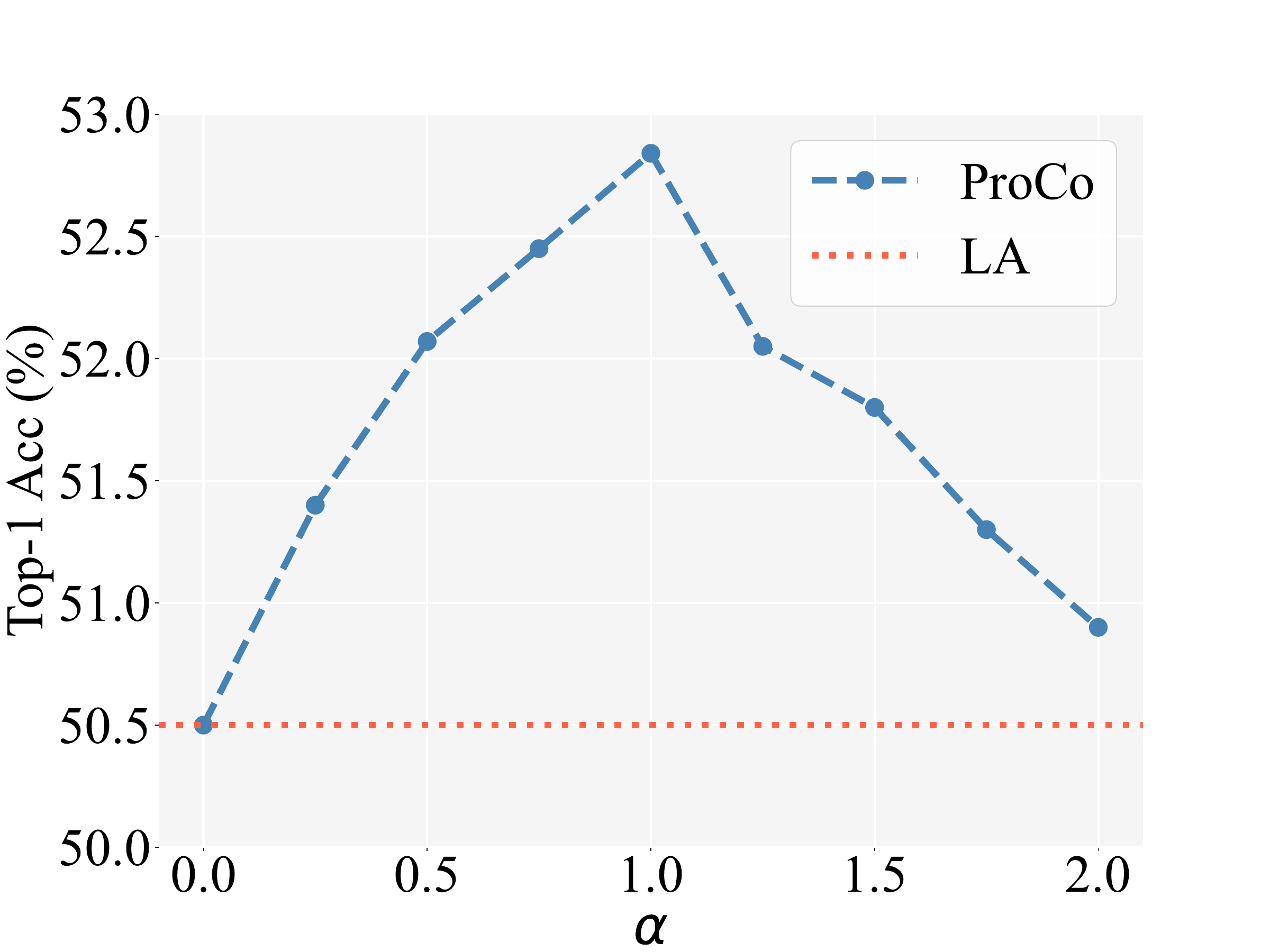}
    \caption{Parameter analysis of the loss weight ratio $\alpha$. LA denotes the logit adjustment method.}
    \label{fig:param}
\end{figure}

\begin{table}[t]
    \caption{Comparison of training the network  with ('w/') and without ('w/o') employing AutoAugment .}
    
    \centering
    \small
    \begin{tabular}{l|c|c}
    \toprule[1pt]
    Method           & w/o AutoAug & w/ AutoAug\\
    \midrule
    Logit Adj.~\cite{Menon2021}      & 49.7       & 50.5 \\
    BCL~\cite{Zhu2022}             & 50.1       & 51.9 \\
    \midrule
    \cellcolor{lightgray!50}\name                         & \cellcolor{lightgray!50}\textbf{50.5}       & \cellcolor{lightgray!50}\textbf{52.8} \\
    \bottomrule[1pt]
    \end{tabular}
    \label{tab:comparison of augmentations}
\end{table}

\begin{table*}[t]
    \centering
    \small
    \tabcolsep=0.5cm
    \caption{Top-1 accuracy of ResNet-32 on CIFAR-100-LT and CIFAR-10-LT.
    $*$ denotes results borrowed from~\cite{Zhou2020}.
    $\dagger$ denotes models trained in the same setting.
    We report the results of 200 epochs.
    }

    \resizebox{0.85\textwidth}{!}{
    \begin{tabular}{l|ccc|ccc}
    \toprule[1pt]
     Dataset                           & \multicolumn{3}{c|}{CIFAR-100-LT} & \multicolumn{3}{c}{CIFAR-10-LT}\\
     \midrule
     Imbalance Factor                  & 100            & 50             & 10             & 100            & 50             & 10\\
     \midrule
     CB-Focal~\cite{Cui2019}      & 39.6   & 45.2 & 58.0 & 74.6 & 79.3 & 87.5 \\
     LDAM-DRW$^*$~\cite{Cao2019} & 42.0   & 46.6 & 58.7 & 77.0 & 81.0 & 88.2 \\
     BBN~\cite{Zhou2020}          & 42.6   & 47.0 & 59.1 & 79.8 & 81.2 & 88.3 \\
     SSP~\cite{Yang2020}          & 43.4   & 47.1 & 58.9 & 77.8 & 82.1 & 88.5 \\
     TSC~\cite{Li2022}            & 43.8   & 47.4 & 59.0 & 79.7 & 82.9 & 88.7 \\
     Casual Model~\cite{Tang2020} & 44.1   & 50.3 & 59.6 & 80.6 & 83.6 & 88.5 \\
     Hybrid-SC~\cite{Wang2021}    & 46.7   & 51.9 & 63.1 & 81.4 & 85.4 & 91.1 \\
     MetaSAug-LDAM~\cite{Li2021}  & 48.0   & 52.3 & 61.3 & 80.7 & 84.3 & 89.7 \\
     ResLT~\cite{Cui2022}         & 48.2   & 52.7 & 62.0 & 82.4 & 85.2 & 89.7 \\
     Logit Adjustment$^\dagger$~\cite{Menon2021}  & 50.5   & 54.9 & 64.0 & 84.3 & 87.1 & 90.9 \\
     BCL$^\dagger$~\cite{Zhu2022}           & 51.9   & 56.4 & 64.6 & 84.5 & 87.2 & 91.1 \\
     \midrule
     \cellcolor{lightgray!50}\name$^\dagger$             & \cellcolor{lightgray!50}\textbf{52.8} & \cellcolor{lightgray!50}\textbf{57.1} & \cellcolor{lightgray!50}\textbf{65.5} & \cellcolor{lightgray!50}\textbf{85.9} & \cellcolor{lightgray!50}\textbf{88.2} & \cellcolor{lightgray!50}\textbf{91.9} \\
    \bottomrule[1pt]
    \end{tabular}
     \label{tab:cifar-lt}
}
\end{table*}

\begin{table}[t]
    \centering
    \small
    \caption{Top-1 accuracy of ResNet-32 on CIFAR-100-LT with an imbalance factor of 100.
    $\dagger$ denotes the models trained in the same setting
    }
   \begin{tabular}{l|ccc|c}
    \toprule[1pt]
     Method                             & Many   & Med        & Few    & All  \\
    \midrule
    \textit{200 epochs} \\
    DRO-LT~\cite{Samuel2021}              & 64.7   & 50.0          & 23.8   & 47.3\\
    RIDE~\cite{Wang2020}                  & {68.1} & 49.2          & 23.9   & 48.0\\
    Logit Adj.$^\dagger$~\cite{Menon2021} & 67.2   & 51.9          & 29.5   & 50.5\\
    BCL$^\dagger$~\cite{Zhu2022}          & 67.2   & \textbf{53.1} & {32.9} & {51.9}\\
    \midrule
    \cellcolor{lightgray!50}\name$^\dagger$                      & \cellcolor{lightgray!50}\textbf{69.0} & \cellcolor{lightgray!50}52.7   & \cellcolor{lightgray!50}\textbf{34.1} & \cellcolor{lightgray!50}\textbf{52.8}\\
    \midrule
    \textit{400 epochs} \\
    PaCo$^\dagger$~\cite{Cui2021}         & -             & -             & -             & 52.0\\
    GPaCo$^\dagger$~\cite{Cui2023}         & -             & -             & -             & 52.3\\
    Logit Adj.$^\dagger$~\cite{Menon2021} & 68.1          & 53.0          & 32.4          & 52.1\\
    BCL$^\dagger$~\cite{Zhu2022}          & 69.2          & 53.1          & 34.4          & 53.1\\
    \midrule
    \cellcolor{lightgray!50}\name$^\dagger$                      & \cellcolor{lightgray!50}\textbf{70.1} & \cellcolor{lightgray!50}\textbf{53.4} & \cellcolor{lightgray!50}\textbf{36.4} & \cellcolor{lightgray!50}\textbf{54.2}\\
    \bottomrule[1pt]
\end{tabular}
 \label{tab:cifar100-100}
\end{table}

\subsection{Implementation Details}

For a fair comparison of long-tailed supervised image classification, we strictly follow the training setting of~\cite{Zhu2022}.
All models are trained using an SGD optimizer with a momentum set to 0.9.

\textbf{CIFAR-10/100-LT.}\quad
We adopt ResNet-32~\cite{He2016} as the backbone network.
The representation branch has a projection head with an output dimension of $12$8 and a hidden layer dimension of $512$.
We set the temperature parameter $\tau$ for contrastive learning to $0.1$.
Following~\cite{Zhu2022,Cui2021}, we apply AutoAug~\cite{Cubuk2019} and Cutout~\cite{DeVries2017} as the data augmentation strategies for the classification branch, and SimAug~\cite{Ting2020} for the representation branch.
The loss weights are assigned equally ($\alpha=1$) to both branches.
We train the network for $200$ epochs with a batch size of $256$ and a weight decay of $4 \times 10^{-4}$.
We gradually increase the learning rate to $0.3$ in the first $5$ epochs and reduce it by a factor of $10$ at the 160th and 180th epochs.
Unless otherwise specified, our components analysis follows the above training setting.
For a more comprehensive comparison, we also train the model for $400$ epochs with a similar learning rate schedule, except that we warm up the learning rate in the first $10$ epochs and decrease it at 360th and 380th epochs.

In order to evaluate the effectiveness of \name~in semi-supervised learning tasks, we partition the training dataset into two subsets: a labeled set and an unlabeled set.
The division is conducted through a random elimination of the labels.
We predominantly adhere to the training paradigm established by DASO~\cite{Oh2022}.
DASO builds its semantic classifier by computing the average feature vectors of each class for generating pseudo-labels.
The semantic classifier in DASO is substituted with our representation branch for training.
The remaining hyperparameters are kept consistent with those in DASO.

\textbf{ImageNet-LT \& iNaturalist 2018.}\quad
For both ImageNet-LT and iNaturalist 2018, we employ ResNet-50~\cite{He2016} as the backbone network following the majority of previous works.
The representation branch consists of a projection head with an output dimension of $2048$ and a hidden layer dimension of $1024$.
The classification branch adopts a cosine classifier.
We set the temperature parameter $\tau$ for contrastive learning to $0.07$.
RandAug~\cite{Cubuk2020} is employed as a data augmentation strategy for the classification branch, and SimAug for the representation branch.
We also assign equal loss weight ($\alpha = 1$) to both branches.
The model is trained for $90$ epochs with a batch size of $256$ and a cosine learning rate schedule.
For ImageNet-LT, we set the initial learning rate to $0.1$ and the weight decay to $5 \times 10^{-4}$.
In addition,  the model is trained for $90$ epochs with ResNeXt-50-32x4d~\cite{xie2017aggregated} as the backbone network and for $180$ epochs with ResNet-50.
For iNaturalist 2018, we set the initial learning rate to $0.2$ and the weight decay to $1 \times 10^{-4}$.

\subsection{Components Analysis}

\textbf{Performance of Representation Branch.}
Since \name~is based on the representation branch, we first analyze the performance of a single branch.
\cref{fig:representation branch} presents the performance curves of different methods without a classification branch as the batch size changes.
We follow the two-stage training strategy and hyper-parameters of SCL~\cite{Khosla2020}.
In the first stage, the model is trained only through the representation branch.
In the second stage, we train a linear classifier to evaluate the model's performance.
From the results, we can directly observe that the performance of BCL~\cite{Zhu2022} and SupCon~\cite{Khosla2020} is significantly limited by the batch size.
However, ProCo effectively mitigates SupCon's limitation on the batch size by introducing the feature distribution of each class.
Furthermore, a comparison of different class complements is presented in~\cref{tab:comparison of class representation}.
Here, EMA and centroid prototype represent that the class representations are estimated by the exponential moving average and the centroid of the feature vectors, respectively.
Normal distribution represents that we employ the normal distributions to model the feature distribution of each class, even though they are in the normalized feature space.
Compared with merely estimating a class prototype, \name~yields superior results by estimating the distribution of features for each class.
Empirical results also confirm the theoretical analysis in~\cref{sec:probabilistic_contrastive_learning} that the vMF distribution is better suited for modeling the feature distribution than the normal distribution.

\textbf{Comparison of More Settings.}
The results are summarized in~\cref{tab:ablation study}, where we compare \name~with several other settings to clarify the effectiveness of relevant components.
As demonstrated in the table, the logit adjustment method fails to achieve improvement in performance when combined with a learnable class representation ($50.5\%$ vs $50.6\%$, $51.8\%$ vs $51.9\%$). 
This empirical evidence suggests a certain degree of equivalence between the two approaches.
Furthermore, incorporating \name~with SupCon loss~\cite{Khosla2020} yields slightly performance degradation ($52.6\%$ vs $52.8\%$).
This result emphasizes the effectiveness of the assumption on feature distribution and parameter estimation in capturing feature diversity.

\textbf{Loss Formulations.}
We analyze the impact of different loss formulations on the performance of \name.
These two formulations are derived from different SupCon losses~\cite{Khosla2020} as shown in~\cref{prop:sample}.
\cref{tab:comparison of loss} shows the results of different loss formulations on CIFAR-100-LT.
We can observe that $\mathcal{L}_\text{in}$ outperforms $\mathcal{L}_\text{out}$.

\textbf{Sensitivity Test.}
To investigate the effect of loss weight on model performance, we conducted experiments on CIFAR-100-LT dataset with an imbalance factor of 100.
The results are presented in~\cref{fig:param}.
It is evident that the model exhibits optimal performance when $\alpha$ is set to $1.0$. 
This suggests that the optimal performance can be achieved by setting the loss weight of both branches equally.
Therefore, we opt not to conduct an exhaustive hyper-parameter search for each dataset. 
Instead, we extrapolate the same loss weight to the other datasets.

\textbf{Data Augmentation Strategies.}
Data augmentation is a pivotal factor for enhancing model performance.
To investigate the impact of data augmentation, we conducted experiments on the CIFAR100-LT dataset with an imbalance factor of 100.
\cref{tab:comparison of augmentations} demonstrates the influence of AutoAug strategy.
Under strong data augmentation, our method achieves greater performance enhancement, suggesting the effectiveness of our approach in leveraging the benefits of data augmentation.

\subsection{Long-Tailed Supervised Image Classification}
\label{sec:supervised}

\textbf{CIFAR-10/100-LT.}\quad
\cref{tab:cifar-lt} presents the results of \name~and existing methods on the CIFAR-10/100-LT.
\name~demonstrates superior performance compared to other methods in handling varying levels of class imbalance, particularly in situations with higher imbalance factors.
With the same training scheme, we mainly compare with Logit Adj.~\cite{Menon2021} and BCL~\cite{Zhu2022}.
Combining the representation branch based on contrastive learning with the classification branch can significantly improve the performance of the model.
Our method further enhances the performance and effectively alleviates the problem of data imbalance.
Moreover, we report longer and more detailed results on the CIFAR-100-LT dataset with an imbalance factor of $100$ in~\cref{tab:cifar100-100}.
With $200$ and $400$ epochs, especially for tail classes, our method has a $1.2\%$ and $2.0\%$ improvement compared to BCL, respectively.
Meanwhile, our method also maintains an improvement on head categories.

\textbf{ImageNet-LT.}\quad
We report the performance of \name~on the ImageNet-LT dataset in~\cref{tab:imagenet-lt} compared with existing methods.
\name~surpasses BCL by $1.3\%$ on ResNet-50 and ResNeXt-50 with 90 training epochs, which demonstrates the effectiveness of our distribution-based class representation on datasets of large scale.
Furthermore, \cref{tab:imagenet-lt-resnet50} lists detailed results on more training settings for ImageNet-LT dataset.
\name~has the most significant performance improvement on tail categories.
In addition to combining with typical classification branches, \name~can also be combined with other methods to further improve model performance, such as different loss functions and model ensembling methods.
We also report the results of combining with NCL~\cite{Li2022}.
NCL is a multi-expert method that utilizes distillation and hard category mining.
\name~demonstrates performance improvements for NCL.

\textbf{iNaturalist 2018.}\quad
\cref{tab:imagenet-lt}~presents the experimental comparison of \name~with existing methods on iNaturalist 2018 over 90 epochs.
iNaturalist 2018 is a highly imbalanced large-scale dataset, thus making it ideal for studying the impact of imbalanced datasets on model performance.
Under the same training setting, \name~outperforms BCL by $1.7\%$.
Furthermore, to facilitate a comparison with state-of-the-art methodologies, an extended training schedule of 400 epochs is conducted.
The results in \cref{tab:imagenet-lt-resnet50}~indicate that \name~is capable of effectively scaling to larger datasets and longer training schedules.

\begin{table}[t]
\caption{Comparisons on ImageNet-LT and iNaturalist 2018 with different backbone networks. $\dagger$ denotes models trained in the same setting.}

\centering
\small

\begin{tabular}{l|cc|c}
\toprule[1pt]
\multirow{2}{*}{Method} &  \multicolumn{2}{c|}{ImageNet-LT} & iNaturalist 2018   \\
                          \cmidrule(r){2-4}
                                      & Res50         & ResX50        & Res50\\
\midrule
    \textit{90 epochs} \\
$\tau$-norm~\cite{Kang2020}      & 46.7          & 49.4          & 65.6\\
MetaSAug~\cite{Li2021}                 & 47.4          & --            & 68.8\\
SSP~\cite{Yang2020}                   & 51.3          & --            & 68.1\\
KCL~\cite{Kang2020}                    & 51.5          & --            & 68.6\\
DisAlign~\cite{Zhang2021}              & 52.9          & 53.4          & 69.5\\
vMF Classifier~\cite{wang2022towards}  &  --           & 53.7          & --\\
SSD~\cite{Li2021a}                     & --            & 53.8          & 69.3\\
ICCL~\cite{Tiong2023}                  & --            & 54.0          & 70.5\\
ResLT~\cite{Cui2022}                   & --            & 56.1          & 70.2\\
GCL~\cite{MengkeLi2022}                & 54.9          & --            & 72.0\\
RIDE~\cite{Wang2020}                   & 54.9          & 56.4          & 72.2\\
Logit Adj.$^\dagger$~\cite{Menon2021}  & 55.1          & 56.5          & --  \\
BCL$^\dagger$~\cite{Zhu2022}           & 56.0          & 56.7          & 71.8\\
\midrule
\cellcolor{lightgray!50}\name$^\dagger$                       & \cellcolor{lightgray!50}\textbf{57.3} & \cellcolor{lightgray!50}\textbf{58.0} & \cellcolor{lightgray!50}\textbf{73.5}\\
 
\bottomrule[1pt]
\end{tabular}

\label{tab:imagenet-lt}
\end{table}

\begin{table}[t]
    \centering
    \small
    \caption{Top-1 accuracy of ResNet-50 on ImageNet-LT and iNaturalist 2018. \\ $\dagger$ and $\ddagger$ denote models trained in the same settings.}

    \resizebox{0.98\linewidth}{!}{
   \begin{tabular}{l|ccc|c}
    \toprule[1pt]
    Method                              & Many   & Med & Few    & All \\
    \midrule
    \textit{180 epochs, ImageNet-LT} \\
    Logit Adj.$^\dagger$~\cite{Menon2021}  & 65.8          & 53.2          & 34.1          & 55.4\\
PaCo$^\dagger$~\cite{Cui2021}                   & 64.4          & \textbf{55.7}          & 33.7          & 56.0\\
   BCL$^\dagger$~\cite{Zhu2022}              & 67.6          & 54.6          & 36.8          & 57.2 \\
   NCL$^\ddagger$~\cite{Li2022}               & 68.2          & 53.9          & 36.3          & 57.0 \\
   NCL (ensemble)$^\ddagger$~\cite{Li2022}    & 69.1          & 56.4          & 38.9          & 59.2 \\
    \midrule
   \cellcolor{lightgray!50}\name$^\dagger$                          & \cellcolor{lightgray!50}\textbf{68.2} & \cellcolor{lightgray!50}{55.1} & \cellcolor{lightgray!50}\textbf{38.1} & \cellcolor{lightgray!50}\textbf{57.8} \\
   \cellcolor{lightgray!50}\name+NCL$^\ddagger$                      & \cellcolor{lightgray!50}\textbf{68.4} & \cellcolor{lightgray!50}\textbf{54.9} & \cellcolor{lightgray!50}\textbf{38.6} & \cellcolor{lightgray!50}\textbf{57.9} \\
   \cellcolor{lightgray!50}\name+NCL (ensemble)$^\ddagger$           & \cellcolor{lightgray!50}\textbf{70.6} & \cellcolor{lightgray!50}\textbf{57.4} & \cellcolor{lightgray!50}\textbf{40.8} & \cellcolor{lightgray!50}\textbf{60.2} \\
   \midrule
   \textit{400 epochs, iNaturalist 2018} \\
   PaCo$^{\dagger}$~\cite{Cui2021}           & 70.3          & 73.2      & 73.6   & 73.2  \\
   BatchFormer$^\dagger$~\cite{Hou2022}      & 72.8       & 75.3      & 75.3    & 75.1 \\
   GPaCo$^{\dagger}$~\cite{Cui2023}          & 73.0      &  75.5      & 75.7     & 75.4  \\
    \midrule
   \cellcolor{lightgray!50}\name$^\dagger$   & \cellcolor{lightgray!50}\textbf{74.0} & \cellcolor{lightgray!50}\textbf{76.0} &  \cellcolor{lightgray!50}\textbf{76.0} & \cellcolor{lightgray!50}\textbf{75.8} \\
   \bottomrule[1pt]
\end{tabular}
}

    \label{tab:imagenet-lt-resnet50}
\end{table}

\begin{table}[t]
    \centering
    \small
    \caption{Top-1 accuracy for ResNet-32 and ResNet-50 on balanced datasets. The ResNet-32 model is trained on CIFAR-100/10 for 200 epochs from scratch. For CUB-200-2011, the pre-trained ResNet-50 model is fine-tuned for 30 epochs.}

\resizebox{.9\linewidth}{!}{
    \begin{tabular}{l|ccc}
    \toprule[1pt]
                               &              CIFAR-100       & CIFAR-10  & CUB-200-2011 \\
     \midrule
        CrossEntropy          & 71.5           & 93.4 & 81.6 \\
     \midrule
        \cellcolor{lightgray!50}\name    &  \cellcolor{lightgray!50}\textbf{73.0} & \cellcolor{lightgray!50}\textbf{94.6} & \cellcolor{lightgray!50}\textbf{82.9} \\
    \bottomrule[1pt]
    \end{tabular}
}

    \label{tab:balanced}

\end{table}

\begin{table*}[t]
    \centering
    \small
    \caption{
        Comparison of accuracy (\%) on CIFAR100-LT under $\gamma_l=\gamma_u$ setup.
        $\gamma_l$ and $\gamma_u$ are the imbalance factors for labeled and unlabeled data, respectively.
        $N_1$ and $M_1$ are the size of the most frequent class in the labeled data and unlabeled data, respectively.
        LA denotes the Logit Adjustment method~\cite{Menon2021}.
    $\dagger$ denotes models trained in the same setting.}

    \begin{tabular}{l|cc|cc}
      \toprule[1pt]
    \multirow{3}{*}{Method} & \multicolumn{2}{c}{$\gamma=\gamma_l=\gamma_u=10$} & \multicolumn{2}{c}{$\gamma=\gamma_l=\gamma_u=20$} \\
    \cmidrule(lr){2-3} \cmidrule(lr){4-5}
        &  $N_1=50$ & $N_1=150$ & $N_1=50$ & $N_1=150$ \\
    & $M_1=400$ & $M_1=300$ & $M_1=400$ & $M_1=300$ \\
    
    \cmidrule(r){1-1} \cmidrule(lr){2-3} \cmidrule(lr){4-5}
    
    Supervised                & 29.6 & 46.9 & 25.1 & 41.2 \\
    ~~ w/ LA~\cite{Menon2021} & 30.2 & 48.7 & 26.5 & 44.1 \\
    
    \cmidrule(r){1-1} \cmidrule(lr){2-3} \cmidrule(lr){4-5}
    
    FixMatch~\cite{Sohn2020}    & 45.2 & 56.5 & 40.0 & 50.7 \\
    ~~w/ DARP~\cite{Kim2020}   & 49.4 & 58.1 & 43.4 & 52.2 \\
    ~~w/ CReST+~~\cite{Wei2021} & 44.5 & 57.4 & 40.1 & 52.1 \\
    ~~w/ DASO$^\dagger$~\cite{Oh2022}  & 49.8 & 59.2 & 43.6 & 52.9 \\
    \cellcolor{lightgray!50}~~w/ ProCo$^\dagger$          & \cellcolor{lightgray!50}\textbf{50.9} & \cellcolor{lightgray!50}\textbf{60.2} & \cellcolor{lightgray!50}\textbf{44.8} & \cellcolor{lightgray!50}\textbf{54.8} \\
    
    \cmidrule(r){1-1} \cmidrule(lr){2-3} \cmidrule(lr){4-5}
    
    FixMatch~\cite{Sohn2020} + LA~\cite{Menon2021} & 47.3 & 58.6 & 41.4 & 53.4 \\
    ~~w/ DARP~\cite{Kim2020}              & 50.5 & 59.9 & 44.4 & 53.8 \\
    ~~w/ CReST+~\cite{Wei2021}            & 44.0 & 57.1 & 40.6 & 52.3 \\
    ~~w/ DASO$^\dagger$~\cite{Oh2022}                      & 50.7 & 60.6 & 44.1 & 55.1 \\
    \cellcolor{lightgray!50}~~w/ ProCo$^\dagger$                     & \cellcolor{lightgray!50}\textbf{52.1} & \cellcolor{lightgray!50}\textbf{61.3} & \cellcolor{lightgray!50}\textbf{46.9} & \cellcolor{lightgray!50}\textbf{55.9} \\
    \bottomrule[1pt]
    \end{tabular}
     \label{tab:ssl}
\end{table*}

\subsection{Balanced Supervised Image Classification}

The foundational theory of our model is robust against data imbalance, meaning that the derivation of \name~is unaffected by long-tailed distributions.
In support of this, we also perform experiments on balanced datasets, as illustrated in \cref{tab:balanced}.
For the CIFAR-100/10 dataset, augmentation and training parameters identical to those used for CIFAR-100/10-LT are employed.
Additionally, we expand our experimentation to the fine-grained classification dataset CUB-200-2011. 
These results demonstrate that while our method, primarily designed for imbalanced datasets, mitigates the inherent limitations of contrastive learning in such contexts, additional experiments also highlight its effectiveness on balanced training sets.
This versatility underlines the strength of our method in addressing not only imbalances in $P(y)$ but also intra-class distribution variances in $P(\bm{z}|y)$.
These aspects correspond to the factors $N_y$ and $\Sigma_y$ in \cref{prop:gen_bound}. Overall, the results imply the broad utility of our approach across diverse datasets.

\subsection{Long-Tailed Semi-Supervised Image Classification}
We present experimental results of semi-supervised learning in~\cref{tab:ssl}.
Fixmatch~\cite{Sohn2020} is employed as the foundational framework to generate pseudo labels, and it is assessed effectiveness in comparison to other methods in long-tailed semi-supervised learning.
We mainly follow the setting of DASO~\cite{Oh2022} except for substituting the semantic classifier based on the centroid prototype of each class with our representation branch for training.
\name~outperforms DASO across various levels of data imbalance and dataset sizes while maintaining the same training conditions.
Specifically, in cases of higher data imbalance ($\gamma=20$) and the ratio of unlabeled data ($N_1=50$), our proposed method exhibits a significant performance enhancement (with LA) of up to $\mathbf{2.8\%}$ when compared to DASO.

\subsection{Long-Tailed Object Detection}

In addition to image classification, we extend \name~to object detection tasks.
Specifically, we utilize Faster R-CNN~\cite{Ren2015} and Mask R-CNN~\cite{he2017} as foundational frameworks, integrating our proposed \name~loss into the box classification branch.
This method was implemented using mmdetection~\cite{Chen2019}, adhering to the training settings of the original baselines.
As depicted in \cref{tab:detection}, our approach yields noticeable enhancements on the LVIS v1 dataset, with both Faster R-CNN and Mask R-CNN demonstrating improved performance across various categories.

\begin{table}[t]
    \centering
    \small
    \caption{Results on different frameworks with ResNet-50 backbone on LVIS v1. We conduct experiments with 1x schedule.}
    \resizebox{0.48\textwidth}{!}{
   \begin{tabular}{l|c|c|ccc|c}
    \toprule[1pt]
         & \name & AP$_b$  &  AP$_r$  & AP$_c$ & AP$_f$  & AP$_m$ \\
    \midrule
    \multirow{2}{*}{Faster R-CNN~\cite{Ren2015}} & \xmark &  22.1 &  9.0 &  21.0 &  29.2 &  -- \\
                                                       &\cellcolor{lightgray!50}\cmark & \cellcolor{lightgray!50}\textbf{24.7} & \cellcolor{lightgray!50}\textbf{15.5} & \cellcolor{lightgray!50}\textbf{24.2} & \cellcolor{lightgray!50}\textbf{29.3} & \cellcolor{lightgray!50}-- \\
    \midrule
    \multirow{2}{*}{Mask R-CNN~\cite{he2017}}     & \xmark &  22.5 &  9.1 &  21.1 &  30.1 &  21.7 \\
                                                  & \cellcolor{lightgray!50}\cmark & \cellcolor{lightgray!50}\textbf{25.2} & \cellcolor{lightgray!50}\textbf{16.1} & \cellcolor{lightgray!50}\textbf{24.5} & \cellcolor{lightgray!50}\textbf{30.0} & \cellcolor{lightgray!50}\textbf{24.7} \\
   \bottomrule[1pt]
    \end{tabular}
    }
    \label{tab:detection}
\end{table}

%% file: conclusion.tex
\section{Conclusion}
\label{sec:conclusion}
In this paper, we proposed a novel probabilistic contrastive (\name) learning algorithm for long-tailed distribution.
Specifically, we employed a reasonable and straightforward von Mises-Fisher distribution to model the normalized feature space of samples in the context of contrastive learning.
This choice offers two key advantages.
First, it is efficient to estimate the distribution parameters across different batches by maximum likelihood estimation.
Second, we derived a closed form of expected supervised contrastive loss for optimization by sampling infinity samples from the estimated distribution.
This eliminates the inherent limitation of supervised contrastive learning that requires a large number of samples to achieve satisfactory performance.
Furthermore, \name~can be directly applied to semi-supervised learning by generating pseudo-labels for unlabeled data, which can subsequently be utilized to estimate the distribution inversely.
We have proven the error bound of \name~theoretically.
Extensive experimental results on various classification and object detection datasets demonstrate the effectiveness of the proposed algorithm.

%% file: appendix.tex
\appendices
\section{Proof of \cref{prop:gen_bound}}

Before presenting the proof of \cref{prop:gen_bound}, we introduce several lemmas essential for the subsequent argument. We begin by proving the asymptotic expansion of the \name~loss.

\begin{lemma}[Asymptotic expansion]
    \label{lem:asymp_expansion}
    The \name~loss satisfies the following asymptotic expansion under the \cref{assumption:p_tau}:
    \begin{equation*}
        \mathcal{L}_{\text{\name}}(y,\bm{z}) \sim  - \log \frac{ \pi_{y} e^{{\bm{\mu}_y^\top \bm z}/{\tau} + {1}/({2\tau^2\kappa_y})}}{\sum_{j=1}^K \pi_j e^{{\bm{\mu}_j^\top \bm z}/{\tau} + {1}/({2\tau^2\kappa_j})}}.
    \end{equation*}

\end{lemma}

\begin{proof} 
    Recall that the \name~loss is defined as:
\begin{equation*}
    {\mathcal{L}}_{\text{\name}} = -\log \left( {\pi_{y_i} \frac{C_p(\tilde \kappa_{y_i})}{C_p(\kappa_{y_i})} } \right) + \log \left( {\sum\limits_{j = 1}^K \pi_j \frac{C_p(\tilde \kappa_j)}{C_p(\kappa_j)} } \right),
\end{equation*}
where $C_p(\kappa)$ is the normalizing constant of the von Mises-Fisher distribution, which is given by
\begin{align*}
    C_{p}(\kappa ) &= {\frac {(2\pi )^{p/2}I_{(p/2-1)}(\kappa )}{\kappa ^{p/2-1}}} \\
    \tilde \kappa_j &= ||\kappa_{j} \bm{\mu}_j +  \bm{z}_i/\tau  ||_2.
\end{align*}

Therefore, we aim to demonstrate that:
\begin{align*}
    \frac{e^{\tilde \kappa_{j}}}{e^{\kappa_{j}}} \frac{\kappa_{j}^{(p-1)/2}}{\tilde \kappa_{j}^{(p-1)/2}} &\sim e^{ {\bm{\mu}_j^\top \bm{z}_i}/{\tau} + {1}/({2\tau^2\kappa_j})}.
\end{align*}
According to the calculation formula, the parameter $\kappa$ is computed by $\frac{\bar{R}(p-\bar R)}{1 - \bar{R}^2}$, 
During the training process, it is observed that the value of $\frac{\bar{R}}{1 - \bar{R}^2} \gg 1$.
Consequently, this implies that $\kappa \gg p$.
Referring to the asymptotic expansion of the modified Bessel function of the first kind for large $\kappa$ relative to $p$ \cite{abramowitz1964handbook}, we have
\begin{equation*}
    I_{(p/2-1)}(\kappa) \sim \frac{e^{\kappa}}{\sqrt{2\pi \kappa}}.
\end{equation*}

Therefore, we have
\begin{align*}
    \frac{C_p(\tilde \kappa_{j})}{C_p(\kappa_{j})} &\sim \frac{e^{\tilde \kappa_{j}}}{e^{\kappa_{j}}} \frac{\kappa_{j}^{(p-1)/2}}{\tilde \kappa_{j}^{(p-1)/2}}. \\
\end{align*}
Given \cref{assumption:p_tau} and $\kappa \gg p$, it follows that $\kappa\tau \ll 1$.
By employing the approximation $(1+x)^\alpha \approx 1 + \alpha{x}$, valid for $x \ll 1$, we obtain:
\begin{align*}
    \tilde \kappa_j &= \kappa_j \sqrt{1 +  {2}\bm{\mu}_j^\top \bm{z}_i /(\kappa_j \tau) + 1/(\kappa_j^2 \tau^2)} \\
                    &\sim \kappa_j \left(1 +  {\bm{\mu}_j^\top \bm{z}_i}/{(\kappa_j \tau)} + 1/({2\kappa_j^2 \tau^2})\right) \\
\end{align*}
and 
\begin{align*}
    \tilde \kappa_j^{(p-1)/2} &\sim \kappa_j^{(p-1)/2} \left(1 +  \frac{(p-1)\bm{\mu}_j^\top \bm{z}_i}{2\kappa_j \tau} + \frac{p-1}{{4\kappa_j^2 \tau^2}}\right).
\end{align*}
Given $\kappa \gg p$, we have:
\begin{align*}
    \frac{\kappa_j^{(p-1)/2}}{\tilde \kappa_j^{(p-1)/2}} &\sim 1.
\end{align*}
Consequently, we establish that:
\begin{align*}
    \frac{e^{\tilde \kappa_{j}}}{e^{\kappa_{j}}} \frac{\kappa_{j}^{(p-1)/2}}{\tilde \kappa_{j}^{(p-1)/2}} &\sim e^{ {\bm{\mu}_j^\top \bm{z}_i}/{\tau} + {1}/({2\tau^2\kappa_j})}
\end{align*}

\end{proof}

\begin{lemma}[Bennett's inequality~\cite{Maurer2009}]
\label{lem:bennett}

Let $Z_{1},\ldots,Z_{n}$ be i.i.d. random variables
with values in $[0,B]$ and let $\delta>0$. Then, with probability
at least $1-\delta$ in $(Z_{1},\ldots,Z_{n})$,
\begin{equation*}
\text{\ensuremath{\mathbb{E}Z}} - \frac{1}{n}\sum_{i=1}^{n}Z_{i} \le \sqrt{\frac{2{\mathbb{V}}Z\ln1/\delta}{n}}+\frac{B\ln(1/\delta)}{3n},
\end{equation*}
where $\mathbb{V}Z$ is the variance of $Z$.
\end{lemma}

\begin{lemma}[Variance inequality]
    \label{lem:var_log_lin}
    Let $\mathcal{L}_{\mathrm{log}}(y,\bm{z}):=\log\left(1+e^{-y(\bm{w}^{\top}\bm{z}+\bm{b})}\right)$ and $\mathcal{L}_{\mathrm{lin}}(y,\bm{z}):=-y(\bm{w}^{\top}\bm{z}+\bm{b})$. Then, for any $y\in\{-1,1\}$,
\begin{equation*}
    \mathbb{V}_{\bm{z}|y}[\mathcal{L}_{\mathrm{log}}(y,\bm{z})]\le\mathbb{V}_{\bm{z}|y}[\mathcal{L}_{\mathrm{lin}}(y,\bm{z})].
\end{equation*}
\end{lemma}

\begin{proof}

Consider the function $f_{y}(z) := \log(1 + e^{-yz})$, where $y \in \{-1, 1\}$. 
The derivative $f'_{y}(z)$ is given by $f'_{y}(z) = -\frac{ye^{-yz}}{1 + e^{-yz}}$, which implies that $\sup_{z} |f'_{y}(z)| \leq 1$.
Consequently, $f_{y}$ is a $1$-Lipschitz function which satisfies the following inequality:
\begin{align*}
    |f_{y}(z) - f_{y}(z')| &\leq |z - z'|, \quad \forall z, z' \in \mathbb{R}.
\end{align*}
Regarding the variance of any real-valued function $h$, it is defined as follows:
\begin{align*}
\mathbb{V}[h(z)] &= \mathbb{E}[h^2(z)] - (\mathbb{E}[h(z)])^2 \\
                 &\leq \mathbb{E}[h^2(z)].
\end{align*}
The above inequalities lead to the conclusion that
\begin{align*}
   & \mathbb{V}_{\bm{z}|y}[\mathcal{L}_{\mathrm{log}}(y,\bm{z})]  \\
=& \mathbb{V}_{\bm{z}|y}[f_{y}(\bm{w}^{\top}\bm{z}+\bm{b})]\\
= & \mathbb{V}_{\bm{z}|y}[f_{y}(\bm{w}^{\top}\bm{z}+\bm{b})-f_{y}(\mathbb{E}_{\bm{z}'|y}[\bm{w}^{\top}\bm{z}'+\bm{b}])] \\
\le &\mathbb{E}_{\bm{z}|y}\left[\left(f_{y}(\bm{w}^{\top}\bm{z}+\bm{b})-f_{y}(\mathbb{E}_{\bm{z}'|y}[\bm{w}^{\top}\bm{z}'+\bm{b}])\right)^{2}\right]\\
\le & \mathbb{E}_{\bm{z}|y}\left[\left(\bm{w}^{\top}\bm{z}+\bm{b}-\mathbb{E}_{\bm{z}'|y}[\bm{w}^{\top}\bm{z}'+\bm{b}]\right)^{2}\right]\\
= &\mathbb{E}_{\bm{z}|y}\left[\left(y(\bm{w}^{\top}\bm{z}+\bm{b})-\mathbb{E}_{\bm{z}'|y}[y(\bm{w}^{\top}\bm{z}'+\bm{b})]\right)^{2}\right]\\
= &\mathbb{V}_{\bm{z}|y}[\mathcal{L}_{\mathrm{lin}}(y,\bm{z})].
\end{align*}
\end{proof}

We are now ready to demonstrate the validity of \cref{prop:gen_bound}.
\begin{proof}
First, we examine the class-conditional \name~loss, denoted as $\mathbb{E}_{\bm{z}|y}\mathcal{L}_{\text{\name}}(y,\bm{z})$. 
For a class label $y \in \{-1, 1\}$, according to \cref{lem:bennett}, we establish that with a probability of at least $1 - \frac{\delta}{2}$, the following inequality holds:
\begin{align*}
    &\mathbb{E}_{\bm{z}|y}\mathcal{L}_{\text{\name}}(y,\bm{z}) - \frac{1}{N_{y}}\sum_{i}\mathcal{L}_{\text{\name}}(y,\bm{z}) \\
    \le&\sqrt{\frac{2\mathbb{V}_{\bm{z}|y}[\mathcal{L}_{\text{\name}}(y,\bm{z})]\ln2/\delta}{N_{y}}}+\frac{B\ln(2/\delta)}{3N_{y}}.
\end{align*}
Incorporating \cref{lem:var_log_lin} and \cref{lem:asymp_expansion}, we obtain:
\begin{align*}
    &\mathbb{E}_{\bm{z}|y}\mathcal{L}_{\text{\name}}(y,\bm{z}) - \frac{1}{N_{y}}\sum_{i}\mathcal{L}_{\text{\name}}(y,\bm{z}) \\
    \le&\sqrt{\frac{2\mathbb{V}_{\bm{z}|y}[\mathcal{L}_{\text{lin}}(y,\bm{z})]\ln2/\delta}{N_{y}}}+ \frac{\ln (2/\delta)}{3N_y} \log(1 + e^{||\bm{w}||_2 - b y}) \\
\end{align*}
where $\mathcal{L}_{\text{lin}}(y,\bm{z})$ is defined as \[-y\left(\frac{(\bm{\mu}_{+1}-\bm{\mu}_{-1})^\top \bm{z}}{\tau} + \frac{\kappa_{-1} - \kappa_{+1}}{2\tau^2\kappa_{-1}\kappa_{+1}} + \log\frac{\pi_{+1}}{\pi_{-1}}\right).\]
Moreover, the variance $\mathbb{V}_{\bm{z}|y}[\mathcal{L}_{\text{lin}}(y,\bm{z})]$ is computed as:
\begin{align*}
    \mathbb{V}_{\bm{z}|y}[\mathcal{L}_{\text{lin}}(y,\bm{z})] &= \mathbb{V}_{\bm{z}|y}[(\bm{\mu}_{+1}-\bm{\mu}_{-1})^\top \bm{z}/\tau )]\\
                                                              &= (\bm{\mu}_{+1} - \bm{\mu}_{-1})^\top \bm{\Sigma}_{y} (\bm{\mu}_{+1} - \bm{\mu}_{-1})/\tau^2,
\end{align*}
where $\bm{\Sigma}_{y}$ represents the covariance matrix of $\bm{z}$ conditioned on $y$.
Consequently, We have thus completed the proof for the conditional distribution's error bound as follows:
\begin{align*}
    & \mathbb{E}_{\bm{z}|y}\mathcal{L}_{\text{\name}}(y,\bm{z}) - \frac{1}{N_{y}}\sum_{i}\mathcal{L}_{\text{\name}}(y,\bm{z}) \\
    \le & \sqrt{\frac{2}{N_y} \bm{w}^\top (\Sigma_y )\bm{w}  \ln \frac{2}{\delta}} + \frac{\ln (2/\delta)}{3N_y} \log(1 + e^{||\bm{w}||_2 - b y}),
\end{align*}
where $\bm w = (\bm{\mu}_{+1} - \bm{\mu}_{-1})/\tau$.

To extend this to the generalization bound across all classes, we apply the union bound. 
Consequently, with a probability of at least $1 - \delta$, the following inequality is satisfied:
\begin{align*}
        \mathbb{E}_{(\bm{z},y)} \mathcal{L}_{\text{\name}}(y,\bm{z}) \le & \sum_{y\in\{-1,1\}} \frac{P(y)}{N_y}\sum_{i} \mathcal{L}_{\text{\name}}(y,\bm{z}_i) \notag \\
        +& \sum_{y\in\{-1,1\}} \frac{P(y)\ln (2/\delta)}{3N_y} \log(1 + e^{||\bm{w}||_2 - b y}) \notag \\
        +&\sum_{y\in\{-1,1\}} P(y) \sqrt{\frac{2}{N_y} \bm{w}^\top (\Sigma_y )\bm{w}  \ln \frac{2}{\delta}},
\end{align*}
where $\bm w = (\bm{\mu}_{+1} - \bm{\mu}_{-1})/\tau$.

\end{proof}

\section{Proof of \cref{prop:excess_risk}}
\label{sec:proof_excess_risk}
The primary approach to proving the excess risk bound involves utilizing the asymptotic expansion of the \name~loss, as detailed in \cref{lem:asymp_expansion}, and its compliance with the 1-Lipschitz property as outlined in \cref{lem:var_log_lin}.

\begin{proof}
    Given \cref{lem:asymp_expansion}, we have
    \begin{align*}
        & \mathbb{E}_{(\bm{z},y)} \mathcal{L}_{\text{\name}}(y,\bm{z};\hat{\bm{\mu}}, \hat{\kappa}) - \mathbb{E}_{(\bm{z},y)} \mathcal{L}_{\text{\name}}(y,\bm{z};{\bm{\mu}}^\star, \kappa^\star) \\
        \sim & \mathbb{E}_{(\bm{z},y)} \left( \log \left(1 + e^{-y(\hat{\bm{w}}^\top \bm{z} + \hat{b})}\right)  - \log \left(1 + e^{-y({\bm{w}^\star}^\top \bm{z} + b^\star)}\right) \right), 
    \end{align*}
    where $\hat{\bm{w}} = \frac{\hat{\bm{\mu}}_{+1} - \hat{\bm{\mu}}_{-1}}{\tau}$ and $\hat{b} = \frac{\hat{\kappa}_{-1} - \hat{\kappa}_{+1}}{2\tau^2\hat{\kappa}_{-1}\hat{\kappa}_{+1}} + \log\frac{\pi_{+1}}{{\pi}_{-1}}$, analogously for $\bm{w}^\star$ and $b^\star$.

    Leveraging the 1-Lipschitz property of $f_y$ from \cref{lem:var_log_lin}, we obtain
    \begin{align*}
        & \mathbb{E}_{(\bm{z},y)} \left( \log \left(1 + e^{-y(\hat{\bm{w}}^\top \bm{z} + \hat{b})}\right)  - \log \left(1 + e^{-y({\bm{w}^\star}^\top \bm{z} + b^\star)}\right) \right) \\
    =& \mathbb{E}_{(\bm{z},y)} \left( f_y(\hat{\bm{w}}^\top \bm{z} + \hat{b}) - f_y({\bm{w}^\star}^\top \bm{z} + b^\star) \right) \\
        \le & \mathbb{E}_{(\bm{z},y)} \left| \hat{\bm{w}}^\top \bm{z} + \hat{b} - {\bm{w}^\star}^\top \bm{z} - b^\star \right|.
    \end{align*}
    Considering the convexity of absolute value under linear transformation and the integral inequality, we deduce
    \begin{align*}
        & \mathbb{E}_{(\bm{z},y)} \left| \hat{\bm{w}}^\top \bm{z} + \hat{b} - {\bm{w}^\star}^\top \bm{z} - b^\star \right| \\
        \le & \mathbb{E}_{(\bm{z},y)} \max_{\bm{z}} \left| \hat{\bm{w}}^\top \bm{z} + \hat{b} - {\bm{w}^\star}^\top \bm{z} - b^\star \right| \\
        = & \mathbb{E}_{(\bm{z},y)} \max \left\{\|\hat{\bm{w}}- \bm{w}^\star\|_2 + |\hat{b} - b^\star|, \|\hat{\bm{w}}- \bm{w}^\star\|_2 - |\hat{b} + b^\star|\right\} \\
        \le & \mathbb{E}_{(\bm{z},y)} \|\hat{\bm{w}}- \bm{w}^\star\|_2 + |\hat{b} - b^\star| \\
        = & \|\hat{\bm{w}}- \bm{w}^\star\|_2 + |\hat{b} - b^\star| \\
        = & \frac{\|\Delta\bm{\mu}_{+1} - \Delta\bm{\mu}_{-1}\|_2}{\tau} + \frac{1}{2\tau^2} \left|\Delta\frac{1}{\kappa_{+1}} - \Delta\frac{1}{\kappa_{-1}}\right| \\
        =& \mathcal{O}(\Delta\bm{\mu}+\Delta\frac{1}{\kappa}),
    \end{align*}
    where $\Delta\bm{\mu} = \hat{\bm{\mu}} - {\bm{\mu}}^\star$, $\Delta\frac{1}{\kappa} = \frac{1}{\hat{\kappa}} - \frac{1}{\kappa^\star}$, $\Delta\bm{\mu}_{+1} = \hat{\bm{\mu}}_{+1} - {\bm{\mu}}^\star_{+1}$, $\Delta\bm{\mu}_{-1} = \hat{\bm{\mu}}_{-1} - {\bm{\mu}}^\star_{-1}$, $\Delta\frac{1}{\kappa_{+1}} = \frac{1}{\hat{\kappa}_{+1}} - \frac{1}{{\kappa}^\star_{+1}}$, and $\Delta\frac{1}{\kappa_{-1}} = \frac{1}{\hat{\kappa}_{-1}} - \frac{1}{{\kappa}^\star_{-1}}$.
 By connecting the above inequalities, the proof is completed.
\end{proof}